\documentclass[11pt]{article}

\usepackage[style = alphabetic,citestyle=alphabetic,maxbibnames=99,backend=biber,sorting=nyt,natbib=true,backref=true]{biblatex}
\DefineBibliographyStrings{english}{%
  backrefpage = {Cited on page},% originally "cited on page"
  backrefpages = {Cited on pages},% originally "cited on pages"
}
\usepackage{fullpage,times,url}

\usepackage{spencercommands}
\usepackage{enumitem} % for referring to assumptions
\DeclareMathSizes{9}{8}{7}{5}
\author{
Spencer Frei \\
UC Berkeley \\
frei@berkeley.edu
\and 
Niladri S. Chatterji  \\
Stanford University \\
niladri@cs.stanford.edu
      \and
Peter L. Bartlett \\
UC Berkeley \\
peter@berkeley.edu
}

\date{\today}

\newcommand\numberthis{\addtocounter{equation}{1}\tag{\theequation}}
\title{\textbf{Random Feature Amplification: \\Feature Learning and Generalization in Neural Networks}}

\begin{document}
\maketitle

\begin{abstract}%   <- trailing '%' for backward compatibility of .sty file
   In this work, we provide a characterization of the feature-learning process in two-layer ReLU networks trained by gradient descent on the logistic loss following random initialization.  We consider data with binary labels that are generated by an XOR-like function of the input features. 
    We permit a constant fraction of the training labels to be corrupted by an adversary.  We show that, although linear classifiers are no better than random guessing for the distribution we consider, two-layer ReLU networks trained by gradient descent achieve generalization error close to the 
    label noise rate. 
    We develop a novel proof technique that shows that at initialization, the vast majority of neurons function as random features that are only weakly correlated with useful features, and the gradient descent dynamics `amplify’ these weak,  random features to strong, useful features. 
\end{abstract}

\section{Introduction}
A number of recent works have developed optimization and generalization guarantees for neural networks in the `neural tangent kernel regime', namely, where the behavior of the neural network can be well-approximated by the linearization of the network around its random initialization~\citep{jacot2018ntk,allenzhu2019convergence,zou2019gradient,du2019-1layer,arora2019exact,soltanolkotabi2017overparameterized}.  Although these works provide a deep understanding of the behavior of neural networks in the early stages of training---where the network parameters are close to their initial values---they fail to capture a number of meaningful characteristics of practical neural networks such as the ability to learn features that differ significantly from those found at random initialization~\citep{fort2020ntk,long2021afterkernel}.  This points to the need for analyses of neural network training that can characterize how gradient descent is able to learn meaningful features.

A remarkable feature of neural networks is that despite their capacity to overfit, when trained by gradient descent they are capable of feature-learning even when there is significant label noise in the training data.   Label noise is a common feature in modern machine learning datasets like ImageNet~\citep{shankar2020evaluatingimagenet}, and moreover, some of the most interesting behaviors of neural networks have been observed when they are trained on datasets with artificially introduced random label noise~\citep{zhang2017rethinkinggeneralization}. 
This points to the importance of theoretically understanding the effect of noisy labels on the neural network training process.  A handful of recent works have sought to understand the training dynamics of neural networks in the presence of noisy labels, but were either restricted to neural networks in the neural tangent kernel (NTK) regime, where feature learning is impossible~\citep{hu2020simple,ji2021earlystopped}; failed to provide generalization guarantees for the resulting network~\citep{li2019labelnoise}; or only applied in settings where linear classifiers perform well~\citep{frei2021provable}.  

\begin{figure}[h]
    \centering
    \includegraphics[width=0.45\textwidth]{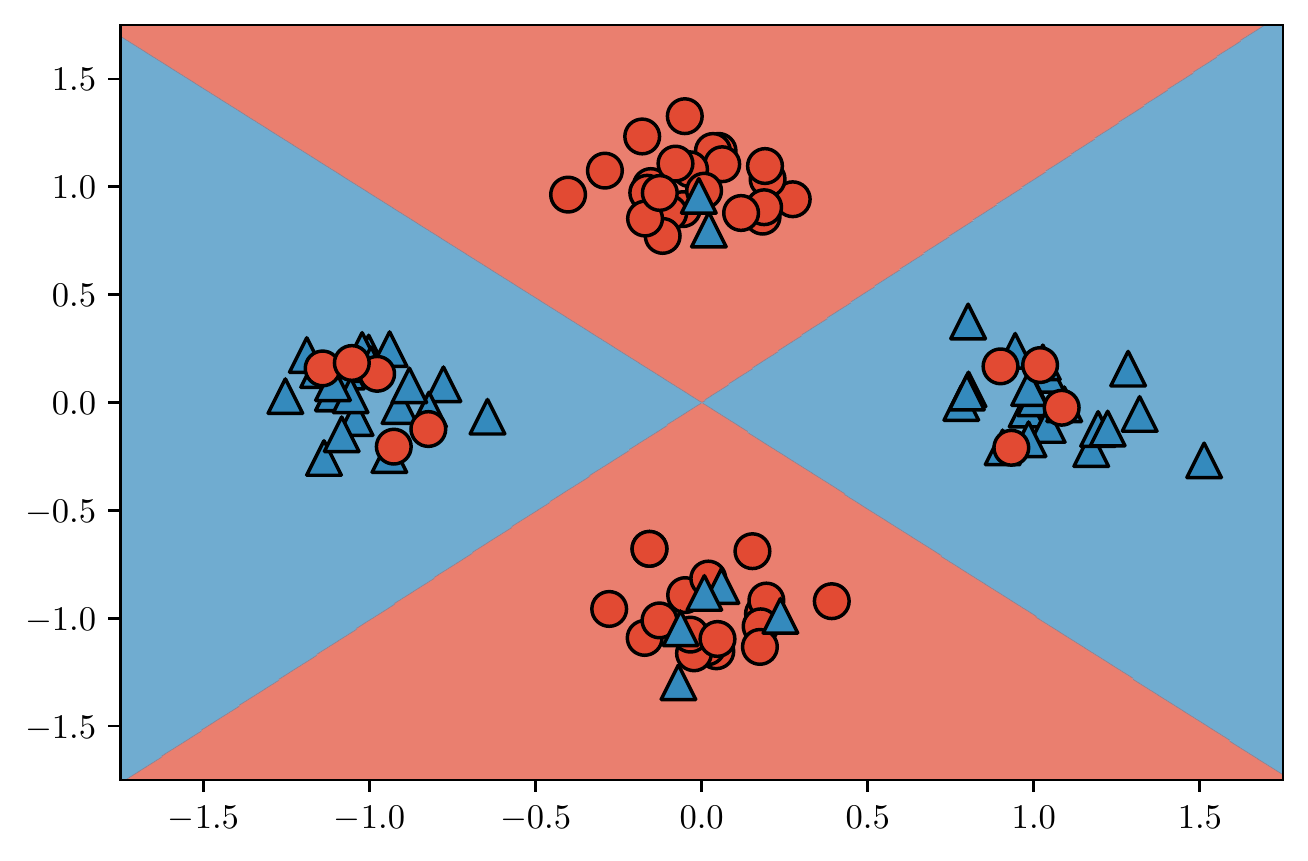}
    \caption{We consider a noisy 2-XOR cluster distribution where opposing cluster means share the same initial `clean' label but a constant fraction of the labels are corrupted by an adversary. The figure is for the special case of Gaussian cluster distributions in $d=2$ dimensions with in-cluster variance $\sigma^2=1/50$ when labels are flipped with probability 15\%.   We plot the decision boundary resulting from training a two-layer ReLU network given $n=5000$ samples (we plot only a subset of the training samples to more clearly illustrate the labels of the samples).  The network was trained for $T=3000$ iterations, with network width $m=500$, step-size $\alpha=0.05$, and initialization variance $\sinit^2 = 1/(32 m)$.}
    \label{fig:noisy.data}
\end{figure}

In this work, we characterize the feature learning process of, and provide generalization guarantees for, two-layer ReLU networks trained by gradient descent on a data distribution where no linear classifier (that use input features) can perform better than random guessing.  In particular, 
we consider two-layer ReLU networks where the first layer is trained while the second layer is fixed at its initial values, and we assume 
the data comes from a uniform mixture of four clusters of data, with means at $+\mu_1, -\mu_1, +\mu_2, -\mu_2$, where $\mu_1,\mu_2\in \R^d$ are orthogonal.  Clean labels are initially generated by an XOR function of the clusters: data from the $+\mu_1$ and $-\mu_1$ clusters have the clean label $+1$, and data from the $+\mu_2$ and $-\mu_2$ clusters have the label $-1$.  We then allow for a constant fraction of these labels to be corrupted arbitrarily.  
Our results show that, provided gradient descent is initialized randomly with a sufficiently small initialization variance and provided the learning rate is sufficiently large, then with high probability gradient descent produces a network that correctly classifies every `clean' test example and incorrectly classifies every `noisy' test example.  
We point the reader to Figure~\ref{fig:noisy.data} to see an example of the data distribution and the decision boundary learned in this setting.  Our results hold for networks of essentially constant width and for arbitrarily small initialization variance. This is in contrast to the neural tangent kernel approaches where the initialization scale is relatively much larger that prevents features to change substantially during training.

Our proof follows by characterizing the types of features that individual neurons learn throughout the training process.  We show that at random initialization, provided the width of the network is a sufficiently large constant, most neurons are `weak' random features: they have a normalized correlation of order $O(1/\sqrt d)$, where $d$ is the input dimension, with at least one of the cluster means $\{\pm \mu_1, \pm \mu_2\}$.  After initialization, provided the learning rate is sufficiently large, a single step of gradient descent amplifies these neurons from `weak' random features to `strong', learned features: the normalized correlations with the cluster means improve from order $O(1/\sqrt d)$ to order $O(1)$.  In the later part of the training process, we show that the gradient descent dynamics ensure that if a neuron is highly correlated with a given cluster center $\mu_s$ after the first step, then (1) its norm increases throughout training, so that the network relies more upon this neuron to determine the network output, and (2) the neuron becomes orthogonal to the opposing cluster center $\mu_{s'}$, $s'\neq s$, so that the neuron is useful only for samples from the cluster center $\mu_{s}$.  We show that having properties (1) and (2) is sufficient for producing a network that classifies all of the clean samples correctly and noisy samples incorrectly.   A key difficulty in showing each of these facts is the presence of noisy training labels, which could in principle prevent the network from learning useful features; a careful analysis shows that this barrier is surmountable provided the fraction of noisy labels is smaller than an absolute constant.

\subsection{Related work}
As mentioned in the previous subsection, a number of works have highlighted the need to develop analyses of neural network training that go `beyond' the NTK, or equivalently, neural networks that lie in the `feature learning regime'. One collection of works has focused on developing separations between what hypothesis classes can be learned efficiently using neural networks in the feature learning regime versus what can be learned using approaches based on kernels or random features~\citep{yehudai2020power,allenzhu.kernel,ghorbani2019limitationslazytraining,wei2019regularization,daniely2020learningparitiesneuralnets,allenzhu2021backward,malach2021quantifying,abbe2021powerofdifferentiable}.  One example of such a hypothesis class includes single neurons $x\mapsto \phi(\sip{w}{x})$, which can be efficiently learned using gradient descent on neural networks beyond the kernel regime~\citep{frei2020singleneuron,yehudai2020singleneuron} but cannot be efficiently learned using random features or kernel-based methods~\citep{yehudai2020power,kamath2020approximategoodenough}.  For a more detailed comparison of recent work on separations between what is learnable using kernel methods versus what is learnable using neural networks in the feature learning regime, we refer the reader to Table 2 and Appendix A of~\citet{malach2021quantifying}.  We note that two concurrent works have shown that a single step of gradient descent suffices for feature-learning behavior in neural networks~\citep{ba2022highdimensionalfeaturelearning,damian2022representationlearning}.  We also show that a single step of gradient descent suffices for learning data-dependent features, but our analysis also requires training for more than one step so that the learned features become more `refined' (see Conditions~\ref{condition:alignment.condition} and \ref{condition:almost.orthogonal} as well as Lemma~\ref{lemma:subnetwork.margin.time.t1} below).

Another line of work utilizes the mean field approximation to connect the training dynamics of infinitely wide neural networks to that of the solution to a partial differential equation~\citep{mei2018meanfield,chizat2018global,wei2019regularization,chen2020generalized,fang2021modelingfeaturesdeepmeanfield}.  This approach allows for the network weights to traverse far from the initialization and learn features.  These works provide a useful characterization of the limiting behavior of neural networks as they become infinitely wide.  By contrast, in this work we provide a guarantee for neural network optimization and generalization for networks of constant width (for a constant level of failure probability).

A handful of other works have explored the behavior of neural networks trained by gradient descent for variants of the XOR distribution we consider in this work.~\citet{wei2019regularization} used the mean field approximation to show that infinite-width two-layer networks trained by gradient flow will generalize well. \citet{bailee20beyondlinearization} considered two-layer neural networks with smooth activations trained with additional `random sign' and $\snorm{W}_{2,4}^8$ penalty regularization.  They showed that when training with a large random initialization and a very large network, the second-order term of the Taylor expansion of the network around its initialization dominates the training dynamics and has a good optimization landscape provided the weights are close enough to initialization.  They used this to derive a generalization guarantee for the resulting network.  Although the work~\cite{bailee20beyondlinearization} is a strict improvement over standard NTK-based approaches, their analysis is more similar to the kernel-based analysis than the feature-learning approach we take here.  Finally,~\citet{daniely2020learningparitiesneuralnets} provided a characterization of learning a noiseless parity over the binary cube when performing gradient descent on the population risk (i.e., assuming infinite samples).  Their analysis relies upon a neuron-by-neuron characterization of the learning process, similar to ours, but it is unclear how their analysis would proceed without access to infinite samples or if there are noisy labels.  Indeed, much of the difficulty in characterizing feature-learning in neural networks comes from the possibility that neural networks could simply memorize the sampled training data rather than learn useful representations that enable generalization to unseen test data.  In contrast to all of the above works, our work provides a novel characterization of how feature-learning occurs in finite-width neural networks that are trained in the finite-sample setting and when a substantial portion of the training labels are adversarially corrupted.

Finally, since our analysis shows that early-stopped gradient descent with a small initialization variance produces neural networks with rather simple decision boundaries which essentially ignore the noisy labels (see Fig.~\ref{fig:noisy.data}), our work is related to a series of works on the simplicity bias of gradient descent~\citep{phuong2021inductivereluorthogonal,lyu2021gdmarginsimplicity,boursier2022gfdynamicsreluorthogonal,frei2023implicit}.  The aforementioned works all rely upon data that is either nearly-orthogonal or exactly orthogonal, while we make no such assumption.  On the other hand, these other works characterize the behavior of gradient descent throughout the entire training trajectory, while we require early-stopping.  
\section{Preliminaries}
%In this section we introduce our notational conventions as well as the problem setting.
We begin with describing our notational conventions. We denote $\norm x$ as the Euclidean norm of a vector $x$.  We will use uppercase letters to refer to matrices, with $\snorm W_F$ denoting the Frobenius norm of a matrix, and $\pnorm W2$ denoting the spectral norm. Given a matrix $W \in \R^{m\times d}$ we let $w_1,\ldots,w_m$ denote the rows of this matrix.
%We use $A = O(B)$ to denote that there exists a universal constant $C\geq 1$ such that $A \leq C B$, and $A = \Omega(B)$ to denote the existence of a constant $C\geq 1$ such that $A \geq B/C$, and $A = \Theta(B)$ to mean that $A$ satisfies both $A = O(B)$ and $A = \Omega(B)$.  
  Given any positive integer $k$, let $[k] = \{1, 2, \ldots, k\}$.% We use the standard ``big Oh notation'' \citep[see, e.g.,][]{cormen2009introduction}.

We next describe the distributional setting.  We consider a joint distribution $\sfP$ over $(x, y)\in \R^d \times \{\pm 1\}$ constructed as follows.
\begin{enumerate}%[topsep=4pt,itemsep=4pt,partopsep=0pt,parsep=0pt]
    \item First, define a cluster distribution $\pclust$ over $\R^d$, which we assume to be log-concave\footnote{That is, $z\sim \pclust$ has a probability density function $p_z$ satisfying $p_z(x) = \exp(-U(x))$ for some convex function $U: \R^d \to \R$.} and satisfies $\E_{z\sim \pclust}[z]=0$ and $\E_{z\sim \pclust}[ z z^\top]=\sigma^2 I_d$ where $\sigma>0$ is a fixed parameter.  
    %Let $\mathsf{Q}$ be a $\rho$-strongly log-concave and isotropic distribution\footnote{That is, $z\sim \mathsf{Q}$ has a probability density function $p_z$ satisfying $p_z(x) = \exp(-U(x))$ for some convex function $U: \R^d \to \R$ such that $\nabla^2 U(x) - \rho I_d$ is positive semidefinite, and satisfies $\E_{z\sim \mathsf{Q}}[z]=0$ and $\E_{z\sim \mathsf{Q}}[\bar z \bar z^\top]$}.  The cluster distribution has samples of the form $\sigma z$, where $z\sim \mathsf{Q}$, and $\sigma>0$ is a fixed parameter.   
    %a Gaussian with mean zero and variance $\sigma^2 I_d$, that is, $\pclust \sim \mathsf{N}(0,\sigma^2 I_d)$.
    \item Let $\mu_1, \mu_2\in \R^d$ be unit norm orthogonal vectors, so $\sip{\mu_1}{\mu_2}=0$ and $\norm{\mu_i}=1$ for $i=1,2$. The positive clusters are centered at $\mu_1$ and $-\mu_1$, while the negative clusters are centered at $\mu_2$ and $-\mu_2$.
    \item %NC
    The distribution of `clean'
    samples $\widetilde \sfP$ is an XOR-like mixture distribution consisting of four independent cluster distributions $\{ \pclust^{(i)}\}_{i=1}^4$ centered at $\mu_1,-\mu_1,\mu_2,-\mu_2$ with labels $+1, +1, -1, -1$ respectively. That is, for example, for $(x,\tilde y) \sim \pclust^{(1)}$, $x = \mu_1 +z$ where $z \sim \pclust$ and $\tilde y = 1$. The 
    %NC
    distribution of clean samples is the uniform mixture 
    $ \widetilde \sfP := \f 1 4 \left[ \pclust^{(1)} +  \pclust^{(2)} +  \pclust^{(3)} +  \pclust^{(4)}\right].$
    \item Finally, the data distribution $\sfP$ is constructed by introducing label noise to $\widetilde \sfP$.  The distribution $\sfP$ has the same marginal distribution over $x$ as $\widetilde \sfP$, but for a given $(x,y)\sim \sfP$, the label $y$ is equal to $\tilde y$ with probability $1-\eta(x)$ and is equal to $-\tilde y$ with probability $\eta(x)$ for some $\eta(x)\in [0,1]$.  We call
    %NC
    $\eta := \E_{x\sim \sfP} \left[\eta(x)\right]$ the \textit{noise rate}.
\end{enumerate}
%\vspace{0.25cm}

We assume the training data $S$ is generated as i.i.d. samples from $\pnoise$,
\[ S := \{(x_i, y_i)\}_{i=1}^{n} \iid \sfP^n.\]
The samples $\{(x_i, y_i)\}_{i=1}^n$ can be partitioned into \textit{clean} and \textit{noisy} samples, where we use the notation $\calct t, \calnt t \subset [n]$ to denote the indices corresponding to the clean and noisy samples.  In particular, using the notation $\tildeyit it$ to denote the clean label for the $i$-th sample, we have
\[  y_i = \begin{cases} \tilde y_i, & i\in \calct t,\\
-\tilde y_i, & i\in \calnt t.\end{cases}\]
We will consider the regime where the noise rate $\eta \approx |\calnt t|/n$ is smaller than
a constant.  In Figure~\ref{fig:noisy.data}, we illustrate what samples from this distribution look like.

We analyze the classification error attained by neural networks trained by gradient descent with the logistic loss given the dataset $S$.  In particular, we consider the class of one-hidden-layer ReLU networks consisting of $m$ neurons with first layer weights $W\in \R^{m\times d}$,
\begin{equation} \label{eq:twolayerrelu}
x\mapsto  f(x; W) := \summ j m a_j \phi(\sip{w_j}{x}), \quad \text{where}\quad \phi(t) := \max\{0,t\}.
\end{equation}
We will use the convention that $\phi$ is applied entry-wise, so that $\phi(Wx)$ has $j$-th component $\phi(\sip{w_j}{x})$. 
For simplicity, we assume that $m$ is an even number and that half of the second layer weights $a_j$ are initialized at the value of $+1/\sqrt m$, and the other half are initialized at the value $-1/\sqrt m$. (Our results hold for odd $m$ by setting $a_m=0$.)  We assume the second layer weights are fixed at their initialized values throughout training. 
This assumption allows for a more simplified analysis as it allows for a static partition of the neurons into `positive' neurons (those for which $a_j>0$) and `negative' neurons ($a_j<0$) throughout training.  We believe it is possible to extend our analysis to the setting where both layers are trained but we do not pursue this question in this work.  
% Due to the homogeneity of the ReLU activation, the hypothesis class of ReLU networks is unchanged by scaling the norm of the second layer weights, which makes this assumption relatively mild.  

Let $\ell(z) := \log(1+\exp(-z))$ be the logistic loss.   We consider the gradient descent algorithm on the empirical risk $\hat L(W)$ corresponding to weights $W$ the $n$ samples $\{(x_i, y_i)\}_{i=1}^n$, where
\[ \hlt t(W) := \f 1n \summ i n \ell\big (\yit  it f(x_i;W)\big).\]
The population risk under the logistic loss is defined as
\[ \quad L(W) := \E_{(x,y)\sim \sfP} \left[\ell\big(y f(x; W)\big)\right].\] 
We consider ReLU networks trained by gradient descent on the first layer weights with fixed learning rate $\alpha>0$ and with random initialization $[\Wt 0]_{i,j} \iid \mathsf{N}(0,\sinit^2)$.  In particular,
\[ 
\Wt {t+1} = \Wt t - \alpha \nabla \hlt t(\Wt t) = \Wt t - \f{\alpha}{n} \summ i n \ell'\big(y_i f(x_i;\Wt t)\big) y_i \nabla f(x_i;\Wt t).
\]
Note that since the ReLU activation $\phi(q)=\max(0,q)$ is not differentiable at 0, we use any subgradient value $\phi'(0)\in [0,1]$  when performing gradient descent.  (Our results do not depend on the value chosen for the subgradient.)
% We will also use the shorthand $-\ell'_{i,t}$ to denote the quantity $-\ell'(y_i f(x_i;\Wt t))$.

We let $C>1$ denote a positive absolute constant that is large enough.
Given a failure probability $\delta \in (0,1/2)$ we make the following assumptions going forward:
\begin{enumerate}[label=(A\arabic*)]
    \item \label{a:dimension}The dimension $d \ge C \max\left\{\log^2(n/\delta), \log(m/\delta)\right\}$;
    \item \label{a:sigma}The in-cluster variance $\sigma^2 \le 1/(C^2d)$;
    \item \label{a:samples}The sample size $n \ge C\log(m/\delta)$;
    \item \label{a:noisy.fraction}The noise rate $\eta \leq 1/C$;
    \item \label{a:width}The number of hidden nodes satisfies $m \ge C\log(1/\delta)$;
    \item \label{a:sinit}The variance at initialization satisfies $0 < \sinit^2 \le \frac{1}{C^4md}$;
    \item \label{a:stepsize}The step-size $\alpha$ satisfies $1/(2\sqrt C) \leq \alpha \leq 1/\sqrt C$.
\end{enumerate}
The first four assumptions above concern the distribution and the relationship between the number of samples, dimension, and number of neurons in the network.  These assumptions are relatively mild as they only require that the dimension and number of samples are logarithmically large.  
These assumptions ensure that the signal-to-noise ratio in the model is quite high, and that in the setting with no label noise $\eta=0$, the optimal test error achievable is $o_n(1)$ (see Appendix~\ref{app:optimal.error} for more details).  
%\footnote{The optimal error is achieved by $x\mapsto \sgn(|\sip{\mu_1}{x}| - |\sip{\mu_2}{x}|)$.  To see that it achieves test error $o_n(1)$, it suffices to show that for $n$ large, there is no overlap between clusters.  For $i\in \{1,2\}$, since $z\sim \pclust$ is log-concave with $\E[zz^\top]=\sigma^2 I$ and $\snorm{\mu_i}=1$, $\sip{\mu_i}{z/\sigma}$ is isotropic and log-concave and hence $\P_{z\sim \pclust}( |\sip{\mu_i}{z}| > 1/4) \leq 3\exp(-\sigma^{-1}/4)$  using \citet[Lemma 5.7]{lovasz}.  By assumptions~\ref{a:sigma} and~\ref{a:dimension} this means $\P_{z\sim \pclust}( |\sip{\mu_i}{z}| > 1/4) \leq 3\exp(-C \sqrt d/4) \leq 3(\delta/n)^{C^{3/2}/4}$. }  
The final three assumptions concern the hyperparameters for the model and the optimization algorithm.  Assumption~\ref{a:width} ensures that the network is wide enough to ensure there are enough random features at initialization for gradient descent to ``amplify''.     
It is noteworthy that assumption~\ref{a:sinit} permits arbitrarily small (but nonzero) initialization variance.  
The assumption~\ref{a:stepsize} ensures that the step-size is large enough so that significant features can be learned after a single step of gradient descent but small enough so that optimization is stable. 

\section{Main results}

Our main contribution is summarized in the following theorem.  
\begin{restatable}{thmv2}{mainthm}\label{thm:final.generalization}
Let $\delta \in(0,1/2)$. 
For all $C>1$ sufficiently large, under the assumptions~\ref{a:dimension} through~\ref{a:stepsize}, by running gradient descent with step-size $\alpha$ for $T = 1+1/(4\alpha)$ iterations, with probability at least $1-4\delta$ 
over the random initialization and the draws of the samples we have,
\begin{enumerate}
    \item For the training points:
    \begin{align*}
    \text{for all $i\in \calct t$,}\quad y_i &= \sgn \l( f(x_i; \Wt {T})\r),\\
     \text{while for all $i\in \calnt t$,}\quad y_i &\neq \sgn\l( f(x_i; \Wt {T})\r).
    \end{align*}
    \item Further, the test error satisfies
    \begin{align*}
        \P_{(x,y)\sim \sfP} \big(y \neq \sgn( f(x; \Wt {T})) \big) &\leq \eta + C \sqrt{\f{\log(1/\delta)}{n}}.
    \end{align*}
\end{enumerate}
%\[ \eta - C \sqrt{\f{\log(1/\delta)}n} \leq  \P_{(x,y)\sim \sfP} \big(y \neq \sgn( f(x; \Wt {T})) \big) \leq \eta + C \sqrt{\f{\log(1/\delta)}{n}}.\]
\end{restatable}
Theorem~\ref{thm:final.generalization} shows that at time $T$, 
%NC
gradient descent learns a network that accurately classifies every clean sample, and incorrectly classifies every noisy sample, and achieves population risk close to the noise rate $\eta$.  In Figure~\ref{fig:noisy.data}, we plot the decision boundary for a neural network trained by gradient descent when 15\% of the training labels are flipped and we observe that indeed every noisy sample is incorrectly classified and every clean sample is correctly classified.  

It is worth noting that the decision boundary displayed in Figure~\ref{fig:noisy.data} is rather simple.  Our proof below will show that this simplicity is due to the fact that nearly every neuron in the neural network will become highly correlated to one of the four cluster means $\{\pm \mu_1, \pm \mu_2\}$ so that the neural network essentially acts as the low-complexity classifier $x\mapsto \sgn(|\sip{\mu_1}{x}|-|\sip{\mu_2}{x}|)$.  The main technical contribution of our work is the characterization of this feature-learning process and an examination of how it proceeds in the presence of noisy labels.

Let us remark that previous works on the generalization of neural networks in the feature-learning regime for variants of the XOR problem we study (without label noise) have sample complexities of order $O(\sqrt{d/n})$, which is an improvement over kernel-based methods which have sample complexity $\Omega(\sqrt{d^2/n})$ \citep{wei2019regularization,bailee20beyondlinearization}.  By contrast, Theorem~\ref{thm:final.generalization} provides a dimension-independent rate of $O(\sqrt{1/n})$.   This difference is due to the fact that 
they consider an XOR problem with a lower signal-to-noise ratio than the one we consider.  In particular, they assume the features are uniform on the hypercube $\{\pm 1 \}^d$ with labels given by $y = \sgn(x_i x_j)$ for distinct coordinates $i\neq j$.  Since the variance in every direction is the same, the signal-to-noise ratio is thus of order $\Theta(1/d)$.  In our setting, the variance in the signal directions is larger: the variance in the direction of $\mu_1$ and $\mu_2$ is equal to $1+\sigma^2$ while the variance in the direction of any vector orthogonal to $\mu_1$ and $\mu_2$ is $\sigma^2$.  Thus, the signal-to-noise ratio in our setting is of order $\Theta\l (\f{1 + \sigma^2}{d\sigma^2}\r) = \Omega(1)$ by Assumption~\ref{a:sigma}.

We note that our analysis does not rely upon the neural tangent kernel approximation.  One way to see this is to observe that the assumption on the width of the network given in Assumption~\ref{a:width} only requires the width to be larger than a fixed constant for a constant level of failure probability.   Moreover, we show explicitly in the following proposition that for each sample, the feature maps given by the hidden layer activations change significantly from their values at random initialization, an essential characteristic of neural networks in the feature-learning regime~\citep{yang2021feature}.   

\begin{restatable}{prop}{featuremapschange}
\label{prop:feature.maps.different}Under the settings of Theorem~\ref{thm:final.generalization}, with probability at least $1-4\delta$ over the random initialization and draws of the samples, the feature maps of the neural network at time $T=1+1/(4\alpha)$ satisfy,
for all $i\in [n]$,
\[  \f{\snorm{\phi(\Wt T x_i) - \phi(\Wt 0 x_i)} }{\snorm{\phi(\Wt 0 x_i)} } \geq \f{1}{C \sinit \sqrt{md}} \geq \f 1 C.\]
In particular, as $\sinit\sqrt{md} \to 0$, the relative change in each sample's feature map is unbounded. 
\end{restatable}
The proof of Proposition~\ref{prop:feature.maps.different} is given in Appendix~\ref{sec:featuremapschange}.

In the next section, we provide the proof of Theorem~\ref{thm:final.generalization}.  The proof follows by concretely characterizing the type of features that different neurons learn throughout the training process.

\section{Proofs}\label{sec:proofs}
In this section, we provide an overview of the proof of Theorem~\ref{thm:final.generalization}. The detailed proofs are collected below in Appendix~\ref{a:main_paper_omitted_proofs}.% Throughout this section we assume that assumptions~\ref{a:dimension}-\ref{a:stepsize} are in force.

We begin by introducing some additional notation that will be needed throughout the proofs. 
As stated above, the set of samples $\{(x_i, y_i)\}_{i=1}^n$ can be partitioned into clean samples and noisy samples, which are identified by the index sets $\calct t, \calnt t \subset [n]$, respectively, and $\calct t \cup \calnt t = [n].$ Each sample comes from one of four clusters, with possible means $\{\pm \mu_1,\pm \mu_2\}$, and we will identify these samples with $\imt {+\mu_1}{t}$, $\imt {-\mu_1}{t}$, $\imt {+\mu_2}{t}$, $\imt {-\mu_2}{t} \subset [n]$.  We further decompose each of these cluster identification sets into the clean and noisy parts, that is, $\imt {+\mu_1}{t}=I_{+\mu_1}^{\calct t} \cup I_{+\mu_1}^{\calnt t}$, and similarly for $I_{-\mu_1}$, $I_{+\mu_2}$, and $I_{-\mu_2}$.  This notation allows for us to write $i\in I_{-\mu_1}$ when we mean $(x_i, y_i) = (-\mu_1 + z, -1)$, where $z\sim \pclust$.   We use the short-hand notation $I_{\pm \mu_1}$ to denote $\imt {+\mu_1}{t}\cup \imt {-\mu_1}{t}$ and likewise for $I_{\pm \mu_2}$.

% Finally, we introduce some notation for the risk under the loss $-\ell'$ which shall be used in a number of places in our proofs:
% \begin{equation}\label{eq:cale.loss}
%     \calE(W) := \E_{(x,y)\sim \sfP}[-\ell'(y f(x; W))],\quad \hat \calE(W) := \f 1 n \summ i n -\ell'(y_i f(x_i;W)).
% \end{equation}
% Since $\ell$ is convex and decreasing, the function $-\ell'$ is non-negative and decreasing, and thus can serve as a surrogate for the zero-one loss.  
We note that there exists a natural neural network consisting of four ReLU neurons that can classify the (clean) data with high accuracy:
\begin{equation} \label{eq:four.neuron.network}
f^\star(x;W) := |\sip{\mu_1}{x}| - |\sip{\mu_2}{x}| = \phi(\sip{\mu_1}{x}) + \phi(\sip{-\mu_1}{x}) - \phi(\sip{\mu_2}{x}) - \phi(\sip{-\mu_2}x).
\end{equation}
This ideal low-complexity classifier is suggestive of the following possibility: for positive neurons, corresponding to second-layer weights satisfying $a_j>0$, the neurons become adapted to either the $+\mu_1$ cluster or the $-\mu_1$ cluster, depending upon the sign of $\sip{\wt 0_j}{\mu_1}$ at initialization.  For negative neurons, corresponding to neurons with $a_j<0$, the neurons become adapted to either the $+\mu_2$ cluster or the $-\mu_2$ cluster depending on the sign of $\sip{\wt 0_j}{\mu_2}$ at initialization.  This is at a high-level the argument that we show below.

In the remainder of this section assume that Assumptions~\ref{a:dimension} through~\ref{a:stepsize} are in force. 
\subsection{Random Initialization and Sample Properties}
We begin with an analysis of the properties of the random initialization. In the lemma below, we derive concentration results on the norm of the random weights, as well as a count for the number of neurons that are correlated with a fixed vector at a given threshold level.  The correlation part of the lemma will be the basis of a `random feature amplification' phenomenon, whereby the relatively small (random) correlations of the neurons with different cluster means at initialization will be amplified into strong correlations by gradient descent.  

\begin{restatable}{lem}{randominit}
\label{lemma:random.init}
Let $\delta \in (0,1/2)$
and let $C_0>1$ be any absolute constant.  Let $\mu \in \R^d$ satisfy $\norm{\mu}=1$.  With probability at least $1-\delta$, we have for all $j \in [m]$,
\[ \f 12 \sinit \sqrt d \leq \snorm{\wt 0_j} \leq \f 32 \sinit \sqrt d,\]
and
 \[\summ j m \ind\l ( |\sip{\wt 0_j}{\mu}| \geq \f{\sinit}{2C_0} \r) \geq m\cdot \l( 1 - \f{1}{2 C_0 } - \sqrt{\f{2\log(4/\delta)}m} \r).\]
\end{restatable}

Recall from~\eqref{eq:four.neuron.network} that there exists a neural network with four ReLU neurons that achieves high accuracy on the clean distribution $\pclean$, with the neuron weights corresponding to the four cluster means $\{\pm \mu_1, \pm \mu_2\}$.  As we noted previously, a potential mechanism for neural network learning would be that most of the positive neurons (with second layer weights $a_j>0)$ become highly correlated with one of the $\pm \mu_1$ clusters while most of the negative neurons become highly correlated with one of the $\pm \mu_2$ clusters.  If $j$-th neuron's weight $w_j$ is highly correlated with a cluster mean $\mu \in \{ \pm \mu_1, \pm \mu_2\}$, then for all samples $x$ coming from the cluster $\mu$, the sign of the activation for a neuron on the sample $\sgn(\sip{w_j}{x})$ would be the same as the activation if the weight were exactly the cluster mean, $\sgn(\sip{\mu}{x})$, so that the $j$-th neuron behaves similarly to the cluster mean $\mu$.     If this occurs we say that the $j$-th neuron \textit{captures} the cluster with mean $\mu$.

We show below that this `capturing' phenomenon can be shown through a two-step process: first, at initialization, most of the positive neurons will have a normalized correlation with $\mu_1$ of order $\Theta(1/\sqrt d)$, and similarly most of the negative neurons will have a normalized correlation with $\mu_2$ of order $\Theta(1/\sqrt d)$.  This is Lemma~\ref{lemma:candidate.subnetwork} below.  Next, we show that by taking a single gradient step with a sufficiently large step-size, the normalized correlations for these neurons will improve from order $\Theta(1/\sqrt d)$ to order $\Theta(1)$.  This result, shown later in Lemma~\ref{lemma:nac.holds.time.1}, is what we refer to as the `random feature amplification' phenomenon, whereby the random features at initialization are amplified into useful features by gradient descent.  
% However, we note that our proof requires that gradient descent is run for a large (but constant) number of steps: although significant feature-learning happens after the first step, in order show the network classifies the clean examples correctly, our proof requires showing a type of feature `refinement' in subsequent steps of gradient descent (see Condition~\ref{condition:almost.orthogonal} and Lemma~\ref{lemma:subnetwork.margin.time.t1} below).  
Towards this end, we characterize the correlations of the neurons with the cluster means at initialization in the following lemma.  

\begin{restatable}{lem}{candidatesubnetwork}\label{lemma:candidate.subnetwork}
Let $\delta \in (0,1/2)$.  For any absolute constant $C_0>1$,
if $C$ is sufficiently large,
with probability at least $1-\delta$ over the random initialization, there exist sets of neurons $J_{+\mu_1}, J_{-\mu_1}, J_{+\mu_2}, J_{-\mu_2}\subset [m]$ satisfying the following: %for every $\mu \in \{ \pm \mu_1, \pm \mu_2 \}$, 
\begin{align*}
\text{for $\mu \in \{\pm \mu_1\}$,}\quad |J_{\mu}| &:= \l |\Bigg \{ j : a_j>0,\ \ip{\f{ \wt 0_j }{\snorm{\wt 0_j}} }{\mu} \geq \f {1}{3C_0 \sqrt d} \Bigg \} \r| \geq \f m 4 \l(1 - \f 1 {C_0}\r)^2,\\
\text{for $\mu \in \{\pm \mu_2\}$,}\quad |J_{\mu}| &:= \l |\Bigg \{ j : a_j<0,\ \ip{\f{ \wt 0_j }{\snorm{\wt 0_j}} }{\mu} \geq \f {1}{3C_0 \sqrt d} \Bigg \} \r| \geq \f m 4 \l(1 - \f 1 {C_0}\r)^2.
    \end{align*}
In particular, $J:= J_{\pm\mu_1}\cup J_{\pm \mu_2}$ satisfies $|J|\geq m(1-1/C_0)^2$. 
\end{restatable}

Lemma~\ref{lemma:candidate.subnetwork} identifies a set of candidate neurons that are partially correlated with the cluster \textit{means} $\{\pm \mu_1, \pm \mu_2\}$.  We would like to translate this result into a statement about the data, and to do so, we first need to provide some basic facts about samples from the distribution.  
The reader may find it helpful to refer back to the beginning of Section~\ref{sec:proofs} where we introduce the $I_{\pm \mu_i}$ notation. 

\begin{restatable}{lem}{datainitialization}\label{lemma:data.initialization}
There is a universal constant $C_1\geq 2$ such that the following holds.  For any $\delta \in (0,1/2)$, for all $C>1$ large enough, with probability at least $1-\delta$ over $S\sim \sfP^{n}$, the following holds. 
\begin{enumerate}%[topsep=4pt,itemsep=4pt,partopsep=0pt,parsep=0pt]
    \item[(a)]  For each $\mu \in \{ \pm \mu_1, \pm \mu_2\}$ and $\mu^\perp$ orthogonal to $\mu$,
    \[\text{for all $i\in I_{\mu}$, }\,\,\, \sip{x_i}{\mu} \geq 1 - C_1 \sigma \sqrt d \geq 1-1/C_1,\quad \text{and}\quad   |\sip{x_i}{\mu^\perp}| \leq C_1 \sigma \sqrt d  \leq 1/C_1.\]
    % while for $\mu^\perp$ orthogonal to $\mu$, we have
    % $$|\sip{x_i}{\mu^\perp}| \leq C_1 \sigma \sqrt d  \leq 1/C_1 .$$  
    \item[(b)] For all $\mu \in \{ \pm \mu_1, \pm \mu_2 \}$, for any $i\in \imt {\mu}{t}$,  $\snorm{x_i - \mu}^2  \leq C_1 \sigma^2 d \leq  1/C_1.$
    \item[(c)] The fraction of noisy points $\frac{|\calnt t|}{n} \leq \eta +C_1 \sqrt{\log(1/\delta) / n} \leq \eta + 1 /C_1$. 
    \item[(d)] For any cluster $\mu \in \{\pm \mu_1, \pm \mu_2\}$ and any $0\leq t \leq T-1$, we have
    \[ \frac 1 4 - C_1 \sqrt{\f{\log(1/\delta)}{n}} \leq \f 1n|\imt {\mu} t| \leq \frac 1 4 + C_1 \sqrt{\f{\log(1/\delta)}{n}}.\]
\end{enumerate}
\end{restatable}
Now, recall that Lemma~\ref{lemma:candidate.subnetwork} shows that a large fraction of the neurons will `capture' at least one of the four cluster centers with a normalized correlation of $\langle w_j^{(0)} / \snorm{\wt 0_j},\mu_s\rangle \geq \Omega(1/\sqrt{d})$.  Since the within-cluster variance is of order $\sigma = O(1/\sqrt d)$, there is not enough signal for these neurons to capture all \textit{samples} within each cluster.  However, the following lemma demonstrates that capturing the cluster mean with a normalized correlation threshold of order $1/\sqrt d$ suffices to guarantee that a strictly larger portion of the samples from that cluster will be captured than not.  This technical lemma will be key to our subsequent analysis. %, and in particular will allow for us to show (in Lemma~\ref{lemma:nac.holds.time.1}) that after one step of gradient descent, the neurons in the candidate subnetwork described in Lemma~\ref{lemma:candidate.subnetwork} can indeed capture all samples within each cluster.  
\begin{restatable}{lem}{njargument}\label{lemma:Nj+.argument.all.times}
There exists a universal constant $C_2>1$ such that for any $\delta \in (0,1/2)$, for all $C>1$ large enough,
with probability at least $1-2\delta$, both Lemma~\ref{lemma:data.initialization} and the following event holds.  For any $j\in [m]$ satisfying $\sip{\wt 0_j / \snorm{\wt 0_j}}{\mu} \geq 1/(3C_0 \sqrt d)$ for some $\mu \in \{\pm \mu_1, \pm \mu_2\}$, it holds that
\[
\sum_{i\in \imct {+\mu}{0}}  \phi'(\sip{\wt 0_j}{\xit i0} ) - \sum_{i\in \imct {-\mu}{0}}  \phi'(\sip{\wt 0_j}{\xit i0} )\geq \f{ n}{C_2}.
\]
\end{restatable}

In light of the above, we introduce the following definition.
\begin{definition}
\label{def:good_run}
We define the event where all parts of Lemma~\ref{lemma:random.init}, Lemma~\ref{lemma:candidate.subnetwork} (with $C_0=4^5 \cdot 1024^2 \exp(4)$),  Lemma~\ref{lemma:data.initialization}, and Lemma~\ref{lemma:Nj+.argument.all.times} hold a \emph{good run}.  
\end{definition}

By the above lemmas, we know that for any $\delta \in (0,1/2)$, for all $C>1$ large enough, a good run occurs with probability at least $1-4\delta$.  In the remainder of this section, we will assume that a good run occurs.  

\subsection{Sufficient Conditions for a Large Margin Classifier via a Good Subnetwork}\label{sec:sufficient.condtiions.for.large.margin}
Our proof will rely upon the notion of a good \textit{subnetwork} of the neural network.  
For index set $\tilde J \subset [m]$ and matrix $W\in \R^{m\times d}$, denote by $W_{\tilde J} \in \R^{|\tilde J|\times d}$ as the sub-matrix of $W$ consisting of rows with indices from $\tilde J$.  Denote by $f^{\tilde J}(x; \cdot)$ the subnetwork consisting of rows from $\tilde J$,
\[ f^{\tilde J}(x;W) := \sum_{j\in \tilde J} a_j \sigma(\sip{w_j}{x}).\]
The below lemma demonstrates that in order to show that the neural network produces a good margin, it suffices to show that there exists a large subnetwork that produces a good margin provided that the weights of the network are bounded.
\begin{restatable}{lem}{goodsubnetworksuffices}\label{lemma:good.subnetwork.suffices}
Let $J\subset [m]$, and denote $J^c = [m] \setminus J$.  If $W\in \R^{m\times d}$ is such that $\snorm{W}_F\leq 1$ and there is a constant $C_f>1$ such that $yf^J(x;W)\geq 1/C_f$ for some $(x,y)\in \R^d\times \{\pm 1\}$, then provided $\norm{x}\leq 2$ and $|J^c|/m \leq 1/(16C_f^2)$, we have 
$yf(x;W) \geq 1/(2C_f)$.
\end{restatable}

Lemma~\ref{lemma:good.subnetwork.suffices} demonstrates that in order to show the neural network classifies an example correctly, it suffices to identify a large subnetwork that does so.  The rest of our proof is dedicated to showing that this happens.  The subnetwork that performs well is defined in terms of the neurons $j \in J_{\pm \mu_1}\cup J_{\pm \mu_2}$, where the index sets $J_{\pm \mu_1}\cup J_{\pm \mu_2}$ are defined in Lemma~\ref{lemma:candidate.subnetwork} and are shown to constitute a large fraction of all of the neurons: for each $\mu \in \{\pm \mu_1, \pm \mu_2\}$, the set $J_{\mu}$ has cardinality at least $|J_{\mu}| \geq \f m 4(1-1/C_0)^2$, where $C_0>1$ is a large constant.  We next define two conditions that we will show suffice for showing this subnetwork classifies examples correctly, which we refer to as the \textit{neuron alignment condition} and the \textit{almost-orthogonality condition}.  We describe the first of these below.
\begin{condition}[Neuron alignment condition]
\label{condition:alignment.condition} We say that the \emph{neuron alignment condition} holds at time $t$ if 
the subsets of neurons $J_{\pm \mu_1}$ and $J_{\pm \mu_2}$ defined in Lemma~\ref{lemma:candidate.subnetwork} satisfy the following: for every $\mu \in \{ \pm \mu_1, \pm \mu_2\}$, and for all $j\in J_{\mu}$,
 \begin{align*}
      \phi'(\sip{\wt {t}_j} {\xit {k}{t}}) &= 1\,\,\, \text{for all } k\in \imt {\mu}{t}, \quad \text{and} \quad \phi'(\sip{\wt {t}_j} {\xit {k}{t}}) = 0\,\,\, \text{for all } k\in \imt {-\mu}{t}.
    \end{align*}
\end{condition}
The neuron alignment condition loosely states that there is a substantial number of neurons (the neurons in the sets $J_{\pm \mu_1}\cup J_{\pm \mu_2}$) that completely capture 
each of the clusters in the sense that all samples within each cluster have the same ReLU activation, which is ``on'' on one of the clusters and ``off'' on the opposing cluster.  By Lemma~\ref{lemma:candidate.subnetwork}, we know that there is a large fraction of neurons that catch the cluster \textit{means} $\{\pm \mu_1, \pm \mu_2\}$ at initialization. However, as we argued prior to Lemma~\ref{lemma:Nj+.argument.all.times}, because the normalized correlation between the neurons at initialization and the cluster means is of order $1/\sqrt d$ while the variance within each cluster is also of order $1/\sqrt d$, a substantial portion of the examples within each cluster will \textit{not} be captured by a neuron at initialization.  We briefly note here that in the next section, we will show that a single step of gradient descent suffices to address this problem.  

We next introduce the notion of \textit{almost-orthogonality}, which will be key to showing that the subnetwork is able to classify examples correctly with a positive margin.   This condition ensures that the $J_{\pm \mu_1}$ neurons capture the $\pm \mu_1$ clusters, and are \textit{almost orthogonal} to data from the $\pm \mu_2$ clusters, and vice versa for the $J_{\pm \mu_2}$ neurons.  In particular, this will allow for us to say that the subnetwork satisfies $f^J(x; W) \approx |\sip{\mu_1}{x}| - |\sip{\mu_2}{x}|$, which one can verify produces a good margin for clean data $(x,\tilde y)\sim \pclean$. 
\begin{condition}[Almost-orthogonality]\label{condition:almost.orthogonal}
We say \emph{almost-orthogonality holds up to time $\tau$} if for all $t\leq \tau$,
\begin{align*}
\text{for all } j\in J_{\pm \mu_1},\quad |\sip{\wt t_j}{\mu_2}| &\leq 3 \alpha |a_j|, \quad \text{and}\quad  \text{for all } j\in J_{\pm \mu_2},\quad  |\sip{\wt t_j}{\mu_1}| \leq 3  \alpha |a_j|.
\end{align*}
\end{condition}
The almost-orthogonality condition ensures that the projection of the $J_{\pm \mu_1}$ (resp. $J_{\pm \mu_2}$) neurons onto the space spanned by $\mu_2$ (resp. $\mu_1$) remains small for all iterates of gradient descent up to time $\tau$.

In the next lemma, we show that the combination of neuron alignment and almost-orthogonality suffices to produce a good subnetwork margin.  Note that we consider times $t \geq 1$ with foresight, as we shall eventually show that neuron alignment and almost-orthogonality hold for all $t \geq 1$. 
%
%We now demonstrate that the combination of the neuron alignment condition and the almost-orthogonality condition suffice for producing a subnetwork with a positive margin.  
\begin{restatable}{lem}{subnetworkmargin}
\label{lemma:subnetwork.margin.time.t1}
Let $J = J_{\pm \mu_1}\cup J_{\pm \mu_2}$, where the sets $J_{\pm \mu_1}$ and $J_{\pm \mu_2}$ are defined in Lemma~\ref{lemma:candidate.subnetwork}.
 Suppose that neuron alignment (Condition~\ref{condition:alignment.condition}) and almost-orthogonality~(Condition~\ref{condition:almost.orthogonal}) hold at times $\tau=1, \dots, T-1=1/(4\alpha)$.
 Then, on a good run,
 for all $C>1$ large enough, at time $T = 1+1/(4\alpha)$, we have $\snorm{\Wt {T}}_F\leq 1$, and that
 \begin{align*} 
\text{ for all } i\in \calct {t},\quad \yit i{t} f^J(\xit i{t}; \Wt {T}) &\geq \frac{1}{C_3} > 0,\quad \text{and}\\
\text{ for all } i\in \calnt {t},\quad \yit i{t} f^J(\xit i{t}; \Wt {T}) &\leq - \frac{1}{C_3} < 0,
\end{align*} 
where $C_3 = 4096 \exp(2) / (1-1/C_0)^2$ and $C_0>1$ is the constant from Lemma~\ref{lemma:candidate.subnetwork}. 
\end{restatable}

Lemma~\ref{lemma:subnetwork.margin.time.t1} demonstrates that in order to show that a given subnetwork $f^J(x; W)$ accurately classifies all of the clean data, it suffices to show that neuron alignment and almost-orthogonality hold for a sufficiently large (but constant) number of steps.  By Lemma~\ref{lemma:good.subnetwork.suffices}, this translates to a guarantee for the entire network $f(x; W)$ if we can show that the subnetwork is sufficiently large, which we can ensure by taking $J$ as the union of the sets $J_{\pm \mu_1},J_{\pm \mu_2}\subset [m]$ as in Lemma~\ref{lemma:candidate.subnetwork} and by taking the constant $C_0>1$ from that lemma to be sufficiently large.  Thus, to complete the proof, we need only verify that neuron alignment and almost-orthogonality hold for a sufficiently large but constant number of steps.  This is what we show in the next subsection.

We note that both neuron alignment and almost-orthogonality are needed in order to ensure that the subnetwork $f^J(x;W)$ behaves like the simple classifier $|\sip{\mu_1}{x}| - |\sip{\mu_2}{x}|$.  For instance, consider what happens if half of the positive neurons (corresponding to $j\in [m]$ with $a_j>0$) are proportional to $\mu_1 + 100\mu_2$ and the other half are proportional to $-\mu_1-100\mu_2$, and likewise half of the negative neurons are proportional to $\mu_2$ and the other half are proportional to $-\mu_2$.  Then the neuron alignment condition would hold, but almost-orthogonality would not hold, and the network would behave like the predictor $|\sip{\mu_1+100\mu_2}{x}| - |\sip{\mu_2}{x}|$ and not generalize well.   Thus, in addition to showing that the neurons are highly correlated with the cluster means from a given class, we must also show that they are nearly orthogonal to the cluster means from the opposite class. 

\subsection{Gradient Descent Produces a Large Margin Classifier}
As mentioned previously, we cannot expect neuron alignment to hold at initialization, as the random features that define the subnetwork $f^J$ have per-neuron normalized correlations $\sip{\wt 0_j/\snorm{\wt 0_j}}{\mu_1}$ of order $O(1/\sqrt d)$, while the fluctuations within each cluster $\sip{\wt 0_j/\snorm{\wt 0_j}}{\mu_s - x_i}$ are also of order $\sigma = O(1/\sqrt d)$.  This means that many samples $x_i$ belonging to a cluster $\mu_s$ will satisfy $\sgn(\sip{\wt 0_j}{x_i})\neq \sgn(\sip{\wt 0_j}{\mu_s})$, preventing the satisfaction of the neuron alignment condition.  This is where Lemma~\ref{lemma:Nj+.argument.all.times} will play a role: although the random features have normalized correlations of order $\Theta(1/\sqrt d)$ with the cluster means, this signal provides an `edge' in terms of the ReLU activations of samples within the cluster.  That is, having $\sip{\wt 0_j/\snorm{\wt 0_j}}{\mu_s} \geq c/\sqrt d$ is sufficient to guarantee that the fraction of samples within the $\mu_s$ cluster sharing the same sign as $\sip{\wt 0_j}{\mu_s}$ is at least $\nicefrac 12 + \Delta$ for some absolute constant $\Delta>0$.  This provides enough signal for gradient descent to latch onto and `amplify' the normalized per-neuron correlations from $\sip{\wt 0_j/\snorm{\wt 0_j}}{\mu_s}\geq c / \sqrt d$ to $\sip{\wt 1_j / \snorm{\wt 1_j}}{\mu_2} \geq c'$ after one sufficiently large step.  Since now the normalized correlations are of order $1$ while the within-cluster fluctuations are of order $1/\sqrt d$, this allows for neuron alignment to hold after a single step of gradient descent.

\begin{restatable}{lem}{nacholds}\label{lemma:nac.holds.time.1}  
% Let $\delta \in (0,1/2)$.  
%Provided the variance at initialization and step-size satisfy $\sinit \sqrt{md}\leq \min(\alpha /3, 1/(16 C_2))$, then
For $C>1$ sufficiently large, on a good run Condition~\ref{condition:alignment.condition} holds at time $t=1$. 
Moreover, letting $C_2>1$ denote the constant from Lemma~\ref{lemma:Nj+.argument.all.times}, the per-neuron normalized correlations satisfy
\[  \text{for every $\mu \in \{ \pm \mu_1, \pm \mu_2\}$ and every $j\in J_{\mu}$,}\quad \ip{\wt 1_j / \snorm{\wt 1_j}}{\mu} \geq \f{1}{ 16 C_2}.\]
%where $C_2>1$ is the constant from Lemma~\ref{lemma:Nj+.argument.all.times}.
% \begin{align*}
% \text{for all } j\in J_{+\mu_1}, \ip{\f{ \wt 1_j}{\snorm{\wt 1_j}}}{+\mu_1} \geq \f{1}{ 16 C_2}; \quad \text{and}\quad \text{for all } j\in J_{-\mu_1}, \ip{\f{ \wt 1_j}{\snorm{\wt 1_j}}}{-\mu_1} \geq \f{1}{ 16 C_2}, \\
% \text{for all } j\in J_{+\mu_2}, \ip{\f{ \wt 1_j}{\snorm{\wt 1_j}}}{+\mu_2} \geq \f{1}{ 16 C_2}; \quad \text{and}\quad \text{for all } j\in J_{-\mu_2}, \ip{\f{ \wt 1_j}{\snorm{\wt 1_j}}}{-\mu_2} \geq \f{1}{ 16 C_2}.
% \end{align*} 
\end{restatable}

We now know that neuron alignment holds at time $t=1$, and that the number of neurons that are characterized by the alignment condition is quite large (precisely, $m(1-1/C_0)^2$ for a large constant $C_0$).  By Lemma~\ref{lemma:subnetwork.margin.time.t1}, if we can show that $(i)$ neuron alignment continues to hold for a certain number of steps, $(ii)$ almost-orthogonality holds throughout these steps, and $(iii)$ we early-stop so that the hidden layer weights are not too large, then there will be a large subnetwork that classifies clean examples with a positive margin.  In the next lemma, we inductively argue that this is the case.

\begin{restatable}{lem}{nacaohold}\label{lemma:nac.orthogonality.holds.all.times}
%For any universal constant $C_0, C_2'>1$, provided the step-size and initialization variance satisfy $\alpha \geq 1/C_2'$ and $\sinit \sqrt{md}\leq \min(\alpha/3, 1/16C_2)$, and $C>1$ is large enough relative to $C_0$ and $C_2'$,
%Old lemma statement
% For $C>1$ sufficiently large,
% on a good run, for every time $t=1,\dots, 1 / (4\alpha)$, neuron alignment (Condition~\ref{condition:alignment.condition}) holds at time $t$, and almost-orthogonality (Condition~\ref{condition:almost.orthogonal}) holds up to time $t$.
%New lemma statement
For $C>1$ sufficiently large,
on a good run, for every time $t=1,\dots, 1 / (4\alpha)$, neuron alignment (Condition~\ref{condition:alignment.condition}) holds at time $t$ and almost-orthogonality (Condition~\ref{condition:almost.orthogonal}) holds up to time $t$.
\end{restatable}

 We emphasize that although Lemma~\ref{lemma:nac.holds.time.1} shows that neuron alignment holds at time $t=1$, this is not sufficient to guarantee generalization since we must ensure that the positive (respectively negative) neurons are not highly correlated to $\pm \mu_2$ (respectively $\pm \mu_1$) since this could result in inaccurate predictions as outlined at the end of Section~\ref{sec:sufficient.condtiions.for.large.margin}.  This potential problem is precisely what almost-orthogonality (Condition~\ref{condition:almost.orthogonal}) prevents, and  Lemma~\ref{lemma:nac.orthogonality.holds.all.times} shows that by running gradient descent for a large (but constant) number of steps, we can guarantee that both neuron alignment and almost-orthogonality hold up to time $T-1=1/(4\alpha)$.   By Lemma~\ref{lemma:subnetwork.margin.time.t1}, this implies that at time $T$ the subnetwork $f^J(x;\Wt T)$ classifies all of the clean examples correctly, and by Lemma~\ref{lemma:good.subnetwork.suffices} this implies that the full network $f(x;\Wt T)$ classifies all of the clean examples correctly with small $\snorm{\Wt T}_F$.  From here the proof of Theorem~\ref{thm:final.generalization} is a straightforward Rademacher-complexity based argument; the details are provided in Appendix~\ref{sec:mainthm}. 
% Below, we provide a proof sketch of the theorem and leave the details for Appendix~\ref{sec:mainthm}.  

% \paragraph{Proof sketch of Theorem~\ref{thm:final.generalization}.}
% First, note that with probability at least $1-4\delta$, a good run occurs, so that the results of Lemmas~\ref{lemma:random.init}-\ref{lemma:Nj+.argument.all.times} all hold. 
% We thus can apply Lemma~\ref{lemma:nac.orthogonality.holds.all.times} so that neuron alignment and almost-orthogonality hold for times $t=1, \dots, 1/4\alpha$. 
% Since neuron alignment and almost-orthogonality hold, by Lemma~\ref{lemma:subnetwork.margin.time.t1}, we know that at time $T$, there exists a subnetwork $f^J(\cdot; \Wt {T})$ that correctly classifies every sample with a clean label and incorrectly classifies every noisy sample.  One can check that the choice of $C_0$ in Definition~\ref{def:good_run} of a `good run' ensures that this subnetwork is large enough so that Lemma~\ref{lemma:good.subnetwork.suffices} holds, so that the entire neural network at time $T$ also accurately classifies every clean example and incorrectly classifies every noisy example, and moreover the weights at time $T$ satisfy $\snorm{\Wt {T}}_F\leq 1$.  Since the weights are bounded and the empirical risk is exactly equal to the number of noisy examples, a straightforward Rademacher-complexity based argument can show that the test error is at most the noise rate plus a term that is of order $O(1/\sqrt n)$. 

\section{Discussion}\label{sec:discussion}
We have shown that two-layer neural networks with ReLU activations trained by gradient descent can achieve small test error on a distribution for which linear classifiers perform no better than random guessing.  We developed a novel proof technique that detailed how using a random initialization provides a collection of random features that gradient descent is able to amplify into stronger, useful features for prediction.  Importantly, our analysis holds when a constant fraction of the training labels are arbitrarily 
corrupted.  

Our analysis requires the usage of early-stopping, so that gradient descent only runs for $T = O(1)$ iterations.  We showed that running gradient descent for $O(1)$ iterations is sufficient to achieve classification error close to the noise rate.  The reason $T = O(1)$ is helpful is that under this assumption, the weights assigned to each sample in the gradient descent updates (proportional to $-\ell'(y_i f(x_i; \Wt t))$) are not too small, so that the useful signals from each sample can be used to push the neural network weights in a good direction.  Early-stopping also allows for a uniform convergence-based argument for the generalization error of the trained network.   Without early-stopping, there is the potential for the neural network to overfit to noisy labels, and it is a natural question whether the network will still generalize near-optimally when it has overfit (i.e., whether or overfitting is `benign' as in~\citet{bartlett2020.benignoverfitting.pnas,frei2022benign}). 

\begin{figure}
    \centering
    %\includegraphics[width=0.475\textwidth]{xor_highn_1e7_v3.pdf}
    %\hfill
    %\includegraphics[width=0.475\textwidth]{xor_highd_1e7_v3.pdf}
    \includegraphics[width=0.8\textwidth]{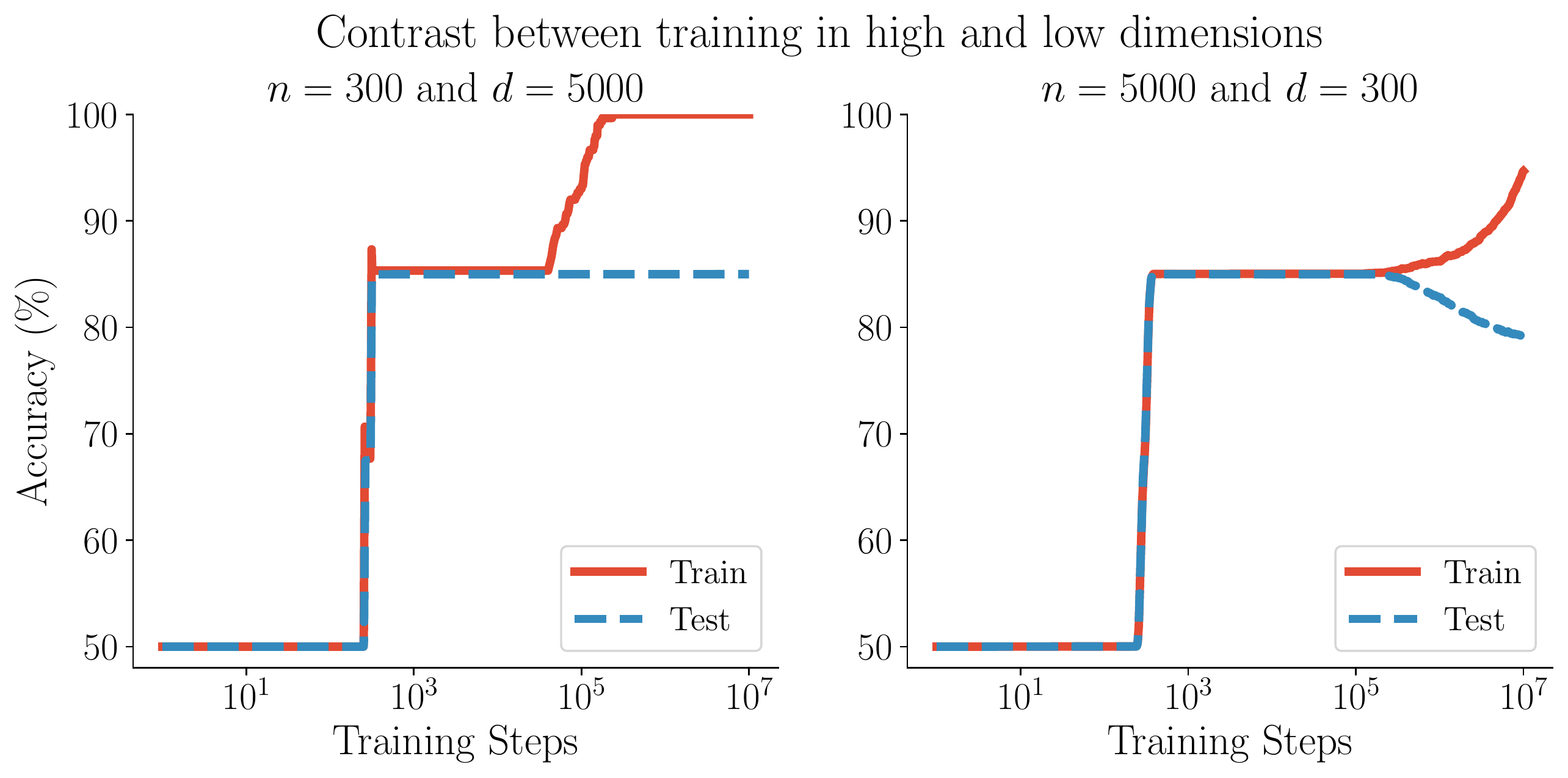}
    \caption{Training and validation accuracy for a two-layer ReLU network with $m=400$ neurons trained on $\pnoise$ (within-cluster variance of $\sigma^2 = 1/d^{1.2}$) when 15\% of the labels within each cluster are flipped to the opposing cluster.  When the network begins to overfit to the noisy labels, the test accuracy decreases in the $n\gg d$ setting while it remains optimal in the $d\gg n$ setting. } %Lines are the average over three random seeds with shaded area (barely visible) corresponding to the minimum and maximum accuracy across the seeds.}
    \label{fig:highd.vs.lowd}
\end{figure}

In Figure~\ref{fig:highd.vs.lowd}, we examine the behavior of two-layer ReLU networks trained by gradient descent on the logistic loss for the 2-XOR distribution we consider when 15\% of the labels are flipped (for full experimental details, see Appendix~\ref{appendix:experiment}).
We consider two distinct settings: a low-dimensional setting where $n\gg d$ and a high-dimensional setting where $d\gg n$.  In the low-dimensional setting, the test accuracy decreases after the network overfits to the noisy training data, while in the high-dimensional setting the test accuracy remains at the optimal 85\% level even after reaching the point of interpolation.  Since our assumptions only require that the number of samples and dimension are not super-exponential in the other, this suggests that we would need to introduce new techniques, separately tailored to the low-dimensional and high-dimensional settings, in order to characterize the generalization behavior of the network after the point of interpolation.  

Another natural direction for future research is to understand whether or not the random feature amplification phenomenon that we identified in two-layer networks has an analogue in deeper networks. Yet another direction is to understand whether this analysis technique can be generalized to settings with more cluster centers.

\section*{Acknowledgements}
We thank the anonymous reviewers for their numerous suggestions which helped improve the presentation of the paper.  We thank Hongren Yan and Yutong Wang for pointing out issues in a previous version of this paper.  We gratefully acknowledge the support of the NSF and the Simons Foundation for
the Collaboration on the Theoretical Foundations of Deep Learning through awards DMS-2023505, DMS-2031883, 
and \#814639.

\appendix
\newpage
{\hypersetup{linkcolor=Black}\tableofcontents}

\section{Omitted Proofs from Section~\ref{sec:proofs}}\label{a:main_paper_omitted_proofs}
In this appendix, we provide the proofs for all of the lemmas in Section~\ref{sec:proofs}.  In Section~\ref{sec:concentration}, we provide the proofs for the lemmas that involve concentration inequalities: Lemmas~\ref{lemma:random.init},~\ref{lemma:candidate.subnetwork},~\ref{lemma:data.initialization}, and~\ref{lemma:Nj+.argument.all.times}.  Next, we prove Lemmas~\ref{lemma:good.subnetwork.suffices} and~\ref{lemma:subnetwork.margin.time.t1}, which show that producing a good subnetwork suffices for the neural network to classify the clean examples correctly and provide a sufficient condition for producing a good subnetwork.  In Section~\ref{sec:gd.produces.good.subnetwork}, we show that gradient descent produces a good subnetwork.  Finally, in Section~\ref{sec:featuremapschange}, we provide a proof of Proposition~\ref{prop:feature.maps.different}, which emphasizes that the feature maps produced by gradient descent differ significantly from those found at initialization.

We remind the reader that throughout this section we assume that Assumptions~\ref{a:dimension}-\ref{a:stepsize} are in effect. We also note that $C>1$ is always used to denote the constant used in these assumptions. 
\subsection{Random Initialization and Sample Properties}\label{sec:concentration}
In this subsection we provide the proofs for Lemmas~\ref{lemma:random.init},~\ref{lemma:candidate.subnetwork},~\ref{lemma:data.initialization}, and~\ref{lemma:Nj+.argument.all.times}.  
\subsubsection{Proof of Lemma~\ref{lemma:random.init}}
We restate the lemma here for the reader's convenience. 
\randominit*
\begin{proof}
We first prove the first part of the lemma.  Note that for fixed $j\in [m]$, there are i.i.d. $z_i\sim \mathsf{N}(0,1)$ such that
\[ \snorm{\wt 0_j}^2 = \summ i d (\wt 0_j)_i^2 = \sinit^2 \summ i d z_i^2 \sim \sinit^2 \cdot \chi^2(d).\]
By concentration of the $\chi^2$ distribution \citep[see,][Example~2.11]{wainwright}, for any $t\in (0,1)$,
\[ \P \l( \l| \f 1 {d\sinit^2} \snorm{\wt 0_j}^2 - 1\r| \geq t\r) \leq 2 \exp(-dt^2/8).\]

In particular, by taking $t = \sqrt{8 \log(4m/\delta)/d}$, we have that for fixed $j\in [m]$, with probability at least $1-\delta/2m$,
\begin{equation*} \label{eq:wj0.norm}
\f 1 2 d \sinit^2 \leq (1 - t) d \sinit^2 \leq \snorm{\wt 0_j}^2 \leq (1 + t) d \sinit^2 \leq \f 3 2 d\sinit^2,
\end{equation*}
where we have used Assumption~\ref{a:dimension}, that is, $d \geq C \log(m/\delta)$% 
for a sufficiently large constant $C>1$ implies $t\leq 1/2$.  Applying a union bound over $j\in [m]$ shows that the bound on the norms at initialization holds over all $j$ with probability at least $1-\delta/2$.

For the neuron-counting argument, let $z\sim \mathsf N(0,1)$ denote a standard normal random variable.  Denote by $p$ the probability
\begin{align*}
 p &:=  \P(|\sip{\wt 0_j}{\mu}| \geq \sinit / (2 C_0)) = \P_{z\sim \mathsf N(0,1)}(|z| \geq 1 / (2 C_0)).
 \end{align*} 
By anti-concentration of the Gaussian, we have
\begin{equation} \label{eq:1-p.lb}
 1-p = \P(|z| \leq 1 / (2 C_0)) \leq 1 / (2 C_0)
 \end{equation} 
Define random variables $U_j := \ind(|\sip{\wt 0_j}{\mu}| \geq \sinit / (2 C_0))$.  Since $U_j$ are 1-sub-Gaussian, Hoeffding's inequality implies for any $t\geq 0$,
\[ \P\l(\l| \summ j m (U_j - p) \r| \geq t\r) \leq 2 \exp (-t^2/2m).\]
Thus, with probability at least $1-\delta/2$, we have
\begin{align*}
\summ j m \ind(|\sip{\wt 0_j}{\mu}| \geq \sinit / (2 C_0)) &= \sum_{j=1}^m U_j \\
&\geq m p- \sqrt{2m \log(4/\delta)} \\
&\overset{(i)}\geq m\cdot \l( 1 - \f{1}{2C_0} - \sqrt{\f{2\log(4/\delta)}m} \r),
\end{align*}
where in $(i)$ we use~\eqref{eq:1-p.lb}.  

Taking a union bound over the first and second parts of the proof shows that both claims hold simultaneously with probability at least $1-\delta$. 
\end{proof}

\subsubsection{Proof of Lemma~\ref{lemma:candidate.subnetwork}}
We restate and prove Lemma~\ref{lemma:candidate.subnetwork} below.
\candidatesubnetwork*

\begin{proof}
Fix $C_0>1$.  Apply Lemma~\ref{lemma:random.init} to the positive neurons $j$ satisfying $a_j>0$ with $\mu_1$.  This tells us that with probability at least $1-\delta/16$, 
\begin{align*}
|J_{\pm \mu_1}| := |\{ j : a_j>0,\ |\sip{\wt 0_j}{\mu_1}| \geq \sinit / 2C_0 \} | &\geq \f m 2 \l( 1 - \f 1 {2C_0} -  \sqrt{\f{ 4 \log(32/\delta)}{m}} \r) \\&\geq \f m 2 (1 - 1/C_0),
\end{align*}
where we have used Assumption~\ref{a:width} so that we may take $m\geq 16 C_0^2 \log(32/\delta)$. 
Notice that for $\wt 0_j\neq 0$, the event $\{ \sip{\wt 0_j}{\mu_1} > 0\} $ depends only on the angle between $\wt 0_j$ and $\mu_1$, while the event $\{ |\sip{\wt 0_j}{\mu_1}|\geq \xi \sinit\}$ depends only on the product $\snorm{\wt 0_j}\snorm{\mu_1}$.  Thus the sign of $\sip{\wt 0_j}{\mu_1}$ is independent of whether or not $j\in J_{\pm \mu_1}$.  Since $\P(\sip{\wt 0_j}{\mu_1}>0)=1/2$, by Hoeffding's inequality we know that with probability at least $1-\delta/16$, at least $\f 12 - \sqrt{\f{4\log(32/\delta)}m}$ fraction of the indices in $J_{\pm \mu_1}$ satisfy $\sip{\wt 0_j}{\mu_1} >0$ and likewise at least $\f 12 - \sqrt{\f{4\log(32/\delta)}m}$ fraction of the indices in $J_{\pm \mu_1}$ satisfy $\sip{\wt 0_j}{\mu_1} <0$.  In particular, taking a union bound we have with probability at least $1-\delta/8$,
\begin{align*}
\left|\left\{ j : a_j>0,\ \sip{\wt 0_j}{\mu_1} \geq \frac{\sinit}{2C_0}\right\} \right| &\geq \f m 4 (1 - 1/C_0) \cdot\l( 1 - \sqrt{\f{4\log(32/\delta)}m}\r) \\ &\geq \f m 4 (1 - 1/C_0)^2,    
\end{align*}
where we have again used Assumption~\ref{a:width}. 
We can argue similarly for neurons satisfying $\sip{\wt 0_j}{-\mu_1}\geq \sinit/2C_0$ and neurons satisfying $\sip{\wt 0_j}{\pm \mu_2}\geq \sinit/2C_0$ to get that with probability at least $1-\delta/2$,
\begin{align*}
     |\{ j : a_j>0,\ \sip{\wt 0_j}{\mu_1} \geq \sinit/2C_0\} | \geq \f m 4 (1 - 1/C_0)^2, \\
    | \{ j : a_j>0,\ \sip{\wt 0_j}{-\mu_1} \geq \sinit / 2C_0\} | \geq \f m 4 (1 - 1/C_0)^2, \\
    | \{ j : a_j<0,\ \sip{\wt 0_j}{\mu_2} \geq \sinit / 2C_0 \}\geq \f m 4 (1 - 1/C_0)^2, \\
    | \{ j : a_j<0,\ \sip{\wt 0_j}{-\mu_2} \geq  \sinit / 2C_0 \} | \geq \f m 4 (1 - 1/C_0)^2.
\end{align*}
By Lemma~\ref{lemma:random.init}, we know that with probability at least $1-\delta/2$, $\snorm{\wt 0_j}\leq \f 32 \sinit \sqrt{d}$, and thus whenever $\sip{\wt 0_j}{\mu_1}\geq \sinit/2C_0$, we have
\[ \ip{\f{\wt 0_j}{\snorm{\wt 0_j}}}{\mu_1} \geq \f{ \sinit}{2 C_0 \snorm{\wt 0_j}} \geq \f{1}{3 C_0 \sqrt d}. \] 
Taking a union bound, we see that with probability at least $1-\delta$,
\begin{align*}
    |J_{+\mu_1}| &:= |\{ j : a_j>0,\ \sip{\wt 0_j / \snorm{\wt 0_j}}{\mu_1} \geq 1/(3 C_0 \sqrt d)\} | \geq \f m 4 (1 - 1/C_0 )^2, \\
    |J_{-\mu_1}| &:=| \{ j : a_j>0,\ \sip{\wt 0_j / \snorm{\wt 0_j}}{-\mu_1} \geq 1/(3 C_0 \sqrt d)\} | \geq \f m 4 (1 - 1/C_0)^2, \\
    |J_{+\mu_2}| &:= | \{ j : a_j<0,\ \sip{\wt 0_j / \snorm{\wt 0_j}}{\mu_2} \geq 1/(3C_0 \sqrt d) \}\geq \f m 4 (1 - 1/C_0)^2, \\
    |J_{-\mu_2}| &:= | \{ j : a_j<0,\ \sip{\wt 0_j / \snorm{\wt 0_j}}{-\mu_2} \geq 1 / (3 C_0 \sqrt d) \} | \geq \f m 4 (1 - 1/C_0)^2.
\end{align*}
% Note that the set of neurons outside of these sets has cardinality at most
% \[ |J^c| = |J \setminus (J_{\pm \mu_1}\ \cup J_{\pm \mu_2})| \leq m\l(1 - (1-2\xi)^2\r).\]
% Thus, for say $\xi=0.001$, at most 1\% of the neurons are not in $J$. 
\end{proof}

\subsubsection{Proof of Lemma~\ref{lemma:data.initialization}}
We restate and prove Lemma~\ref{lemma:data.initialization} below. 
\datainitialization*

\begin{proof} 
We shall show that each part of the lemma holds with a large enough probability and then take a union bound to establish our claim.

\textit{Proof of parts~(a) and (b):} We consider the case $i\in \imt {+\mu_1}{t}$.  The cases of $i\in \imt {\mu}{t}$ for $\mu \in \{-\mu_1, \pm \mu_2\}$ follow using identical arguments. 

Let $i\in \imt {\mu_1}t$.  We begin by noting that, since $\snorm{\mu}=1$, we have by Cauchy--Schwarz,
\begin{align*}
    \sip{x_i}{\mu_1} &= \sip{x_i-\mu_1}{ \mu_1 } + \sip{\mu_1}{\mu_1}\\
    &\geq 1 - \snorm{x_i-\mu}.\numberthis \label{eq:xixk.correlation.positive}
\end{align*}
Therefore, to derive a lower bound on $\sip{x_i}{\mu_1}$ when $i\in\imt {+\mu_1}t$, it suffices to derive an upper bound on $\snorm{x_i-\mu}$ for each $i$, so that we will first prove part (b).  

Since $(x_i-\mu)/\sigma$ is isotropic and log-concave, by concentration of the Euclidean norm of isotropic log-concave random vectors~\cite[Theorem 1]{adamczak2014paouris}, there is a universal constant $c>0$ such that,
\[ \P( \snorm{(x_i - \mu) / \sigma} \geq c u \sqrt d) \leq  \exp(-c u \sqrt d).\]
In particular, using Assumption~\ref{a:dimension}, we can take $d \geq \log^2(32n/\delta)/c^2$ so that $\exp(-cu \sqrt d) \leq \delta/(32n)$ and thus we have with probability at least $1-\delta/32$, for all $i\in \imt {+\mu_1}t$, 
 \[ \snorm{x_i  - \mu} \leq c \sigma \sqrt d.\]
This, along with Assumption~\ref{a:sigma} proves part (b).  

Using~\eqref{eq:xixk.correlation.positive} and Assumption~\ref{a:sigma} so that $c \sigma \sqrt d < 1$, we have
\begin{align*}
    \sip{x_i}{\mu_1} \geq 1 - c \sigma \sqrt d,
\end{align*}
which proves the first half of part (a) of the lemma when $i\in \imt{+\mu_1}{t}$.  When $i\in \imt{+\mu_1}t$, the cluster mean $\mu^\perp$ orthogonal to $\mu_1$ is $\mu_2$, and so we have,
\begin{align*}
    |\sip{x_i}{\mu_2}| &= |\sip{x_i-\mu_1}{\mu_2 }| 
    \leq \snorm{x_i - \mu_1}
    \leq c \sigma \sqrt d,
\end{align*}
which completes the proof of the second part of (a) when $\mu = \mu_1$.  Taking a union bound over $\mu \in \{\pm \mu_1, \pm \mu_2\}$ shows that parts (a) and (b) hold with probability at least $1-\delta/8$. 

\textit{Proof of part~(c):} We note that $\{\ind(y_i \neq \tildeyit it)\}_{i=1}^n$ are a collection of $n$ i.i.d. random variables bounded by one with expectation equal to the noise rate $\eta$.  For some absolute constant $c>1$, Hoeffding's inequality therefore gives, for any $u \geq 0$,
\[ \P\l(\f 1 n\summ i n \ind(y_i \neq \tildeyit it) - \eta \geq u \r) = \P\l(\f{|\calnt t|}n - \eta \geq u\r) \leq \exp\l( - \f{n u^2}{2 c}\r).\]
In particular, for $u = \sqrt{\f{2c\log(2/\delta)}n}$, by Assumption~\ref{a:samples} we have with probability at least $1-\delta/2$,  $|\calnt t|/n \leq \eta + \sqrt{2c \log(2/\delta)/n} \leq \eta + 1/ C_1$ by Assumption~\ref{a:noisy.fraction}.   

%Taking a union bound over this event and the events in part $(a)$ and $(b)$ shows that with probability at least $1-3\delta/4$ all three of $(a)$, $(b)$, and $(c)$ hold. 

\textit{Proof of part~(d):} We consider the case $\mu = +\mu_1$ with identical arguments holding for $\mu \in \{-\mu_1, \pm \mu_2\}$. Notice that 
the random variables $\{\ind(i\in \imt {+\mu_1}{t} )\}_{i=1}^n$ are i.i.d. Bernoulli with mean $1/4$, since the samples  are drawn uniformly from the four clusters $\{\pm \mu_1, \pm \mu_2 \}$. 
Thus, by Hoeffding's inequality, for some absolute constant $c>1$, we have with probability at least $1-\delta/16$,
\begin{equation}\label{eq:Imu1.lb}
    \f 1 4 - c\sqrt{\f{\log(32/\delta)}n} \leq \f 1 n |\imt {+\mu_1} 0| \leq \f 1 4 + c\sqrt{\f{\log(32/\delta)}n}.
\end{equation}
Since $\delta \in (0,1/2)$, there is a larger constant $c'>0$ such that $c \sqrt{\frac{\log(32/\delta)}{n}}\leq c' \sqrt{\frac{\log(1/\delta)}{n}}$. Taking a union bound  over the four clusters shows that part (d) holds with probability at least $1-\delta/4$. 

Thus all four parts (a), (b), (c), (d) hold with probability at least $1-\delta$. 
\end{proof}

\subsubsection{Proof of Lemma~\ref{lemma:Nj+.argument.all.times}}
We restate and prove Lemma~\ref{lemma:Nj+.argument.all.times} below. 
\njargument*
\begin{proof}We shall prove this lemma in two parts. First, we shall define a ``good event'' $\calE$ that occurs with probability at least $1-2\delta$. Then via a deterministic argument, we shall show that the lemma holds whenever this good event occurs.

%We consider the case $\mu = +\mu_1$, with identical arguments holding for $\mu \in \{-\mu_1, \pm \mu_2\}$.   
\paragraph{Defining the good event.}
Fix some $\mu \in \{\pm \mu_1, \pm \mu_2\}$.   By definition of $\pnoise$, there are $\zit it  \iid \pclust$  such that
%SF: pclust isn't N(0,I_d)
\begin{align*}
    N_ {\mu}(j) &:= \sum_{i\in \imct {\mu}{t}} \phi'(\sip{\wt 0_j}{x_i}) = \sum_{i\in \imct {\mu}{t}} \ind(\sip{\wt 0_j}{\mu} + \sigma \sip{\zit it}{\wt 0_j} > 0).
\end{align*}
Similarly, there are $\uit it \iid \pclust$ such that
    \begin{align*} 
    N_{-\mu}(j)  &:= \sum_{i\in \imct {-\mu}{t}} \phi'(\sip{\wt 0_j}{x_i}) = \sum_{i\in \imct {-\mu}{t}} \ind(-\sip{\wt 0_j}{\mu} + \sigma \sip{\uit it}{\wt 0_j} > 0).
    \end{align*}
Thus, if we define,
\[p := \P_{z\sim \pclust}(\sip{ \wt 0_j}{\mu} + \sigma \sip{z}{\wt 0_j} > 0),\]
then we have,
\begin{equation*}  \label{eq:Nj.distribution}
N_{\mu}(j) \sim \Binomial(|\imct {\mu}{t}|, p)\quad\text{and}\quad N_ {-\mu}(j) \sim \Binomial(|\imct {-\mu}{t}|, 1-p).
\end{equation*} 
This motivates deriving upper and lower bounds for the cardinality of the sets $\imct{\mu}{t}$ and $\imct{-\mu}{t}$.   To do so, we first note that with probability at least $1-\delta$, all of the events in Lemma~\ref{lemma:data.initialization} hold.  In particular, by Part~(d) of that lemma, we have with probability at least $1-\delta$,
for any $\mu \in \{\pm \mu_1, \pm \mu_2\}$,
\begin{equation}\label{eq:Imu1.lb.alltimes}
    \f 1 4 - C_1\sqrt{\f{\log(1/\delta)}n} \leq \f 1 n |\imt {\mu} t| = \f 1 n (|\imct {\mu}{t }| + |\imnt {\mu}{t }|) \leq \f 1 4 + C_1\sqrt{\f{\log(1/\delta)}n}.
\end{equation}
We thus have with probability at least $1-\delta$, all of the events in Lemma~\ref{lemma:data.initialization} hold, and, for all $\mu \in \{\pm \mu_1, \pm \mu_2\}$,
\begin{align*}
&\f n 8 \overset{(i)}\leq \f n 4 \l( 1 - C_1 \sqrt{\f{\log(1/\delta)}n} - \f{|\calnt t|}{n} \r) \leq |\imct {\mu}{t}| \leq \f n 4 \l( 1 + C_1 \sqrt{\f{\log(1/\delta)}n}\r), \numberthis \label{eq:Icleanmu1.lb.alltimes}
\end{align*}
where the inequality $(i)$ uses Assumptions~\ref{a:samples} and~\ref{a:noisy.fraction}.  
% In the remainder of the proof, we will condition on the event that the results in Lemma~\ref{lemma:data.initialization} hold, so that~\eqref{eq:Icleanmu1.lb.alltimes} holds at all times $t$.  

Now, since $N_{\mu}(j) \sim \Binomial(|\imct {\mu}{t}|, p)$, by Hoeffding's inequality and a union bound (over the neurons and over the clusters), there is some $c>0$ such that with probability at least $1-\delta$, for all $\mu \in \{\pm \mu_1,\pm \mu_2\}$, for all $j\in [m]$,
\begin{equation} \label{eq:Nj.lb.intermediate.alltimes}
|\imct {\mu}{t}| \cdot \l( p - c \sqrt{\f{  \log(64m/\delta)}{|\imct {\mu}{t}|}} \r) \leq N_{\mu}(j) \leq   |\imct {\mu}{t}| \cdot \l( p + c \sqrt{\f{  \log(64m/\delta)}{|\imct {\mu}{t}|}}  \r).
\end{equation}
Let us define $\calE$ to be the event where the events in Lemma~\ref{lemma:data.initialization}, and inequalities~\eqref{eq:Icleanmu1.lb.alltimes} and~\eqref{eq:Nj.lb.intermediate.alltimes} all simultaneously hold. By a union bound this happens with probability at least $ 1-2\delta$. This shall determine the success probability of the lemma. 

\paragraph{Lemma holds whenever the \emph{good event} $\calE$ occurs.} In the remainder of the proof let us assume that this event $\calE$ occurs; we will show that the lemma holds as a deterministic consequence of these events.

Since the event $\calE$ occurs, for all $\mu \in \{\pm\mu_1,\pm\mu_2\}$ and all $j \in [m]$ we have,
\begin{align}\nonumber 
    N_{\mu} (j)&\overset{(i)}\geq |\imct {\mu}{t}|\l( p - c\sqrt{\f{\log(64m/\delta)}{|\imct {\mu}{t}|}}\r)  \\ \nonumber 
    &\overset{(ii)}\geq \f n4 \l(p - 3c \sqrt{\f{\log(64m/\delta)}n} \r)\l( 1 - C_1 \sqrt{\f{\log(1/\delta)}n} - \f{|\calnt t|}n\r)\\
    &\overset{(iii)}\geq \f n4 \l[ \l( 1 - \f{|\calnt t|}n\r)p - 4c \sqrt{\f{\log(64m/\delta)}n} \r].\label{eq:nplusmu1.alltimes}
\end{align}
Above, $(i)$ uses Eq.~\eqref{eq:Nj.lb.intermediate.alltimes}, while $(ii)$ uses Eq.~\eqref{eq:Icleanmu1.lb.alltimes}.  Inequality $(iii)$ uses Assumption~\ref{a:samples} so that $n \geq \log(64m/\delta)$ and by taking $c$ to be a larger constant.    Further, for all $\mu \in \{\pm \mu_1,\pm \mu_2\}$ and all $j \in [m]$, we have,
\begin{align}\nonumber 
    N_ {-\mu}(j) &\leq| \imct {-\mu}{t}|\l( (1-p) + c \sqrt{\f{\log(32 m/\delta)}{|I_{-\mu}|}} \r)  \\ \nonumber 
    &\overset{(i)}\leq \f n 4 \l( 1 + C_1 \sqrt{\f{\log(1/\delta)}n}\r) \l( (1-p) + 3c \sqrt{\f{\log(32 m/\delta)}{ n}} \r)  \\ 
    &\overset{(ii)}\leq \f n4  \l[ \l( 1 + 4 C_1 \sqrt{\f{\log(1/\delta)}n} \r)(1-p) + 4c \sqrt{\f{\log(64m/\delta)}n} \r] . \label{eq:nminusmu1.alltimes}
\end{align}
Above, $(i)$ uses~\eqref{eq:Icleanmu1.lb.alltimes} and $(ii)$ uses the assumption $n\geq C\log(1/\delta)$ given by~\ref{a:samples}.  Thus, we have shown that, when the good event $\calE$ occurs, then inequalities~\eqref{eq:nplusmu1.alltimes} and~\eqref{eq:nminusmu1.alltimes} hold. In the remainder of the proof, we will show that the lemma follows as a  consequence of the inequalities~\eqref{eq:nplusmu1.alltimes} and~\eqref{eq:nminusmu1.alltimes}.

In order to show $N_{\mu}(j) \gg N_{-\mu}(j)$, it suffices to show that $p$ is large enough so that there is sufficient `edge' for more samples to be captured by $w_j$ than not.  To this end, we have for any $j$ such that $\sip{\wt 0_j}{\mu}>0$,
%SF: z_1 is not N(0,1), but is projection of log-concave distribution to one d space; same argument holds 
\begin{align*} 
p &= \P_{z\sim \pclust} (\sip{\wt 0_j}{\mu} + \sigma \sip{z}{\wt 0_j} > 0) \\
&= \P_{z\sim \pclust}\l(  \ip{z}{\f{\wt 0_j}{\snorm{\wt 0_j}}}> -\f{  \sip{\wt 0_j}{\mu}}{\sigma \snorm{\wt 0_j}} \r)\\
&= \f 12 + \P_{z\sim \pclust}\l( \ip{z}{\f{\wt 0_j}{\snorm{\wt 0_j}}} \in \l[ - \f{\sip{\wt 0_j}{\mu}}{\sigma \snorm{\wt 0_j}} , 0\r] \r).\numberthis \label{eq:nj.p.definition.alltimes}
\end{align*} 
Recall that we are considering neurons $j\in [m]$ such that $\sip{\wt 0_j / \snorm{\wt 0_j}}{\mu} \geq 1/(3C_0\sqrt d)$. By assumption~\ref{a:sigma}, for $C$ sufficiently large we have $\sigma \sqrt d \leq 1/C \leq 3/C_0$ so that the inclusion $[-1/9, 0] \subset [-1/(3C_0 \sigma \sqrt d), 0]$ holds.   Thus, we have,
\begin{align*} 
p &\geq \f 12 + \P_{z\sim \pclust}\l( \ip{z}{\f{\wt 0_j}{\snorm{\wt 0_j}}} \in \l[ - \f{1}{3 C_0 \sigma \sqrt d} , 0\r] \r)\\
&\geq \f 12 +\P_{z\sim \pclust}\l( \ip{z}{\f{\wt 0_j}{\snorm{\wt 0_j}}} \in \l[ - \f 1 9, 0 \r] \r).
\end{align*}
Note that $\sip{z}{\wt 0_j/\snorm{\wt 0_j}}$ is the projection of a log-concave isotropic random vector onto the one dimensional subspace spanned by $\wt 0_j/\snorm{\wt 0_j}$, and thus by~\citet[Definition 1.2, Fact A.4]{diakonikolas2020massartstructured} there exists an absolute constant $c_1>0$ such that
\begin{equation}\label{eq:anticoncentration.logconcave}
\P_{z\sim \pclust} \l( \ip{z}{\f{\wt 0_j}{\snorm{\wt 0_j}}} \in \l[ - \f 1 9, 0 \r] \r) \geq c_1,
\end{equation}
and continuing from the previous display we thus have
\begin{align*}
    p\geq \f 12 + c_1. 
\label{eq:p0.lowerbound.nj+} \numberthis 
\end{align*}

%In $(i)$, we have used Lemma~\ref{lemma:random.init}

%Now, if we had random classification noise, we would be able to say that $|\imct {\mu}{0}|$ is proportional to $1/4\times \eta$, where $\eta$ is the noise rate.  But since the label noise could be adversarial, the adversary could corrupt up to $\eta$ fraction of the data, which leaves the possibility that $|\imct {\mu}{0}|$ could be approximately $1/4$ of the data or could be $1/4-\eta$ of the data.  On the other hand, we know 

We can thus use the inequalities given in events~\eqref{eq:nplusmu1.alltimes} and~\eqref{eq:nminusmu1.alltimes} to see that for any $\mu \in \{\pm \mu_1, \pm \mu_2\}$ and for any $j\in [m]$ such that $\sip{\wt 0_j / \snorm{\wt 0_j}}{\mu} \geq 1/(3C_0\sqrt d)$, 
\begin{align*}
   & N_{\mu}^{(0)}(j) - N_{-\mu}^{(0)}(j) \\&\qquad \geq \f n4 \l[ \l(1 - \f{|\calnt 0|}n\r)p - \l(1 + 4C_1 \sqrt{\f{\log(64m/\delta)}n} \r)(1-p) - 8c \sqrt{\f{\log(64m/\delta)}n}\r] \\
    &\qquad \geq \f n4 \l[ \l(2 - \f{|\calnt 0|}n\r) p -1 - 10c \sqrt{\f{\log(64m/\delta)}n} \r]\\ 
    &\qquad \overset{(i)}\geq\f n4 \l[ \l(2 - \f{|\calnt 0|}n\r) \l( \f 12 + c_1  \r)  -1 - 10c \sqrt{\f{\log(64m/\delta)}n} \r]\\ 
    &\qquad \overset{(ii)}\geq\f n4 \l[ 2c_1   -  \f{ |\calnt 0|}n - 10c \sqrt{\f{\log(64m/\delta)}n} \r]\\ 
    &\qquad \overset{(iii)} \geq\f n4 \cdot c_1.\numberthis 
    \label{eq:nmu1.minus.nminusmu1}
\end{align*}
In $(i)$, we have used~\eqref{eq:p0.lowerbound.nj+}.  Inequality $(ii)$ follows by a direct calculation.  Finally, $(iii)$ uses that Assumption~\ref{a:samples} ensures $n \geq 4 \cdot 100c^2 c_1^{-2} \log(64m/\delta)$, as well as Lemma~\ref{lemma:data.initialization}(c) and Assumption~\ref{a:noisy.fraction}.

This shows that 
%, conditioned on the events in Lemma~\ref{lemma:data.initialization},
there exists a universal constant $C_2>1$ such that whenever event $\calE$ occurs, for all $\mu \in \{\pm \mu_1,\pm \mu_2\}$ and  \begin{align}\text{ for all }j\text{ such that } \sip{\wt 0_j/\snorm{\wt 0_j}}{\mu}\geq 1/(3C_0 \sqrt d),\quad & N_{\mu}^{(0)}(j) - N_{-\mu}^{(0)}(j) \geq \frac{n}{C_2}.\label{eq:Nj0.result}\end{align}

\paragraph{Putting things together.}Recall that above we argued that the event $\calE$ (which is when the events in Lemma~\ref{lemma:data.initialization} and Equations~\eqref{eq:Icleanmu1.lb.alltimes} and~\eqref{eq:Nj.lb.intermediate.alltimes} all hold simultaneously) occurs with probability at least $1-2\delta$, and since this implies the claim in Equation~\eqref{eq:Nj0.result} holds, this completes the proof. 
\end{proof}

\subsection{Sufficient Conditions for a Large Margin Classifier via a Good Subnetwork}\label{sec:good.subnetwork.suffices}
In this subsection, we prove Lemmas~\ref{lemma:good.subnetwork.suffices} and~\ref{lemma:subnetwork.margin.time.t1}, which demonstrate that in order to show the neural network correctly classifies all clean samples, it suffices to show that there exists a large subnetwork that classifies the points correctly.  Before we prove this, we introduce the following auxiliary lemma, which bounds the growth of the weights of the network over time.  This lemma will be used in a number of places in the remaining proofs.

\subsubsection{Auxiliary Lemma on Neuron Weight Growth}
\begin{restatable}{lem}{perneuronnorm}\label{lemma:per.neuron.norm}
%Provided $\sinit \sqrt{md} \leq \alpha/3$, 
For $C>1$ large enough, on a good run
we have the following bound on the norms of the weights for times $t\geq 1$:
\begin{enumerate}
    \item For all  $j\in [m]$, $\snorm{\wt t_j} \leq 2|a_j| \alpha  t$;
    \item $\snorm{\Wt t}_F\leq 2 \alpha t$.
\end{enumerate}
\end{restatable}

% \perneuronnorm*

\begin{proof}
First, note that since a good run occurs, 
Lemma~\ref{lemma:data.initialization} and Assumption~\ref{a:sigma} imply that for any sample $i\in [n]$, we have $\snorm{x_i - \mu_s}^2 \leq C_1 \sigma^2 d + \f{1}{C_1} \leq 1/3$, where $\mu_s$ is the cluster mean corresponding to $x_i$.  Therefore, we have for any $i\in [n]$,
\begin{equation}\label{eq:xi.norm.bound}
    \snorm{x_i} \leq (1 + 1/\sqrt 3) \leq \sqrt 2 .
\end{equation}
We can thus bound,
\begin{align*}
 \snorm{\wt t_j - \wt 0_j} &\leq \alpha \summm \tau 0 {t-1} \snorm{\nabla_j \hlt \tau (\Wt \tau)} \\
     &= \alpha \summm \tau 0 {t-1} \norm{\f 1 n \summ i n -\ell'_{i,\tau } \yit i\tau a_j \phi'(\sip{\wt \tau_j}{\xit i\tau}) \xit i\tau} \\
     &\leq \alpha \summm \tau 0 {t-1} \f 1 n \summ i n |\ell'_{i,\tau }| |a_j| \phi'(\sip{\wt \tau_j}{\xit i\tau}) \snorm{\xit i\tau} \\
     &\leq \alpha t |a_j| \sqrt{2},\numberthis \label{eq:wtj-w0j.norm}
\end{align*}
where the final inequality uses~\eqref{eq:xi.norm.bound} and $|\ell'_{i,\tau}|\le 1$. Therefore, by the triangle inequality and Lemma~\ref{lemma:random.init},
\begin{align*}
    \snorm{\wt t_j}&\leq \snorm{\wt 0_j} + \sqrt{2} |a_j|\alpha t \\
    &\leq \f 32 \sinit \sqrt{d} +  \sqrt{2} |a_j| \alpha t\\
    &\leq 2 |a_j| \alpha t,
\end{align*}
where the final inequality uses Assumptions~\ref{a:sinit} and~ \ref{a:stepsize} so that $\sinit \sqrt{md}\leq \alpha/3$ for $C>1$ sufficiently large.  The bound on the Frobenius norm follows by noting that $\snorm{\Wt t}_F^2 = \summ j m \snorm{\wt t_j}^2$ and that $|a_j| = 1/\sqrt{m}$. 
\end{proof}

\subsubsection{Proof of Lemma~\ref{lemma:good.subnetwork.suffices}}
With the above lemma in hand, we now restate and prove Lemma~\ref{lemma:good.subnetwork.suffices}. 
\goodsubnetworksuffices*

\begin{proof}
By definition,
\[ f(x;W) = f^J(x;W) + f^{J^c}(x;W) = \sum_{j\in J} a_j \phi(\sip{w_j}{x}) + \sum_{j\in J^c} a_j \phi(\sip{w_j}{x}).\]
For the latter term, note that
\begin{align*}
    |f^{J^c}(x;W)| &= \left|\sum_{j\in J^c} a_j \phi(\sip{w_j}{x})\right| \\
    &\overset{(i)}\leq \sqrt{\sum_{j\in J^c} a_j^2} \sqrt{\sum_{j\in J^c} \sip{w_j}{x}^2} \\
    &= \sqrt{\f{|J^c|}{m}} \snorm{W_{J^c}x}_2 \\
    &\leq \sqrt{\f{|J^c|}{m}} \snorm{W_{J^c}}_2 \snorm{x}\\
    &\leq \sqrt{\f{|J^c|}{m}} \snorm{W_{J^c}}_F \snorm{x}.
\end{align*}
In $(i)$ we use the Cauchy--Schwarz inequality, and that $\phi$ is 1-Lipschitz with $\phi(0)=0$.  The final claim follows as $\snorm{W_{J^c}}_F\leq \snorm{W}_F\leq 1$, so that
\[ f(x; W) \geq f^J(x; W) - \sqrt{ \f{|J^c|}m} \snorm{W_{J^c}}_F \snorm{x} \geq \f 1 {C_f} - \f 1{4 C_f} \cdot 1 \cdot 2 = \f 1 {2 C_f}.\]
\end{proof}

\subsubsection{Proof of Lemma~\ref{lemma:subnetwork.margin.time.t1}}
In this section we restate and prove Lemma~\ref{lemma:subnetwork.margin.time.t1}. 

\subnetworkmargin*

\begin{proof}
By Lemma~\ref{lemma:per.neuron.norm}, we have that for all $\tau  \in \{1,\ldots, T\}$
\begin{equation*}
\snorm{\Wt {\tau}}_F \leq 2 \alpha \tau \leq 1,
\end{equation*}
since $\tau \le T=1/(4\alpha)+1$.
This shows the claimed guarantee for the norm.   

We now show the claim for the margin.  First, note that we have for any $x\in \R^d$ and $W\in \R^{m\times d}$, since $\phi$ is 1-Lipschitz, Cauchy--Schwarz gives
\begin{align}\label{eq:nn.output.bound.spectral.norm}
    |f(x; W)| &= \l| \summ j m a_j \phi(\sip{w_j}{x})\r| \leq \sqrt{\summ j m a_j^2} \sqrt{\summ j m \sip{w_j}{x}^2} = \snorm{Wx} \leq \snorm{W}_2\snorm{x}.
\end{align}
Using Lemma~\ref{lemma:data.initialization}(b), we can therefore bound the neural network output at time $\tau$ by
\begin{equation*}
|f(\xit i\tau; \Wt \tau)| \leq \snorm{\Wt \tau}_F \snorm{x_i} \leq 2,\quad \text{for all } i\in [n],\ \tau \leq T-1.
\end{equation*}
Note that $-\ell'(z)$ is a decreasing function and also that $-\ell'(z) \geq \nicefrac 12 \exp(-z)$ on $z \geq 0$. Therefore,
\begin{equation}\label{eq:loss.lb}
 -\ell'(\yit i\tau  f(\xit i\tau ;\Wt \tau)) \geq \frac 12 \exp(-2),\quad \text{for all } i\in [n],\ \tau \leq T-1.
\end{equation}

We will now show that for sufficiently large $t$, the network produces a positive margin on the $+\mu_1$. This shall be crucial in showing that the network produces a positive margin on the clean points associated with this cluster, and a negative margin on the noisy points in the cluster. 

Recall the notation $-\ell'_{i,t} = -\ell'( y_i f(x_i; \Wt t))$.
Since neuron alignment holds, we have for $j\in J_{+\mu_1}$ and $\tau \leq T-1$,
\begin{align*}
   \sip{\wt {\tau+1}_j - \wt \tau_j}{\mu_1} &=\frac{\alpha |a_j|}{n}\sum_{i=1}^n -\ell'_{i,\tau} \yit i\tau \phi'(\langle w_j^{(\tau)}, \xit i\tau \rangle) \sip{\xit i\tau }{\mu_1}\\
        & = \frac{\alpha |a_j|}{n}\sum_{i\in \imct {+\mu_1}{\tau}} -\ell'_{i,\tau} \sip{ \xit i\tau }{\mu_1}  - \frac{\alpha |a_j|}{n}\sum_{i\in \imnt {+\mu_1}{\tau }} -\ell'_{i,\tau} \sip{ \xit i\tau }{\mu_1} \\
    &\qquad -\frac{\alpha |a_j|}{n}\sum_{i\in \imt {+\mu_2}{\tau}} -\ell'_{i,\tau} \yit i\tau \phi'(\langle w_j^{(\tau )}, \xit i\tau\rangle) \sip{\xit i\tau}{\mu_1} \\
    &\qquad -\frac{\alpha |a_j|}{n}\sum_{i\in \imt {-\mu_2}{\tau}} -\ell'_{i,\tau } \yit i\tau \phi'(\langle w_j^{(\tau )},\xit i\tau \rangle) \sip{\xit i\tau}{\mu_1}\\
    &\overset{(i)}\geq \frac{\alpha |a_j|}{n}\sum_{i\in \imct {+\mu_1}{\tau}} -\ell'_{i,\tau} \cdot \f 12  - \frac{\alpha |a_j|}{n}\sum_{i\in \imnt {+\mu_1}{\tau }} -\ell'_{i,\tau} \cdot \f 32  \\
    &\qquad -\frac{\alpha |a_j|}{n}\sum_{i\in \imt {+\mu_2}{\tau}} -\ell'_{i,\tau} \cdot C_1 \sigma \sqrt d -\frac{\alpha |a_j|}{n}\sum_{i\in \imt {-\mu_2}{\tau}} -\ell'_{i,\tau } \cdot C_1 \sigma \sqrt d\\
    &\overset{(ii)}\geq \alpha |a_j| \cdot \l[ \f{ |I_{+\mu_1}^\calC|}{n} \cdot \f{\exp(-2)}{4} - \f{ |I_{+\mu_1}^\calN|}{n} \cdot \f 3 2 - \f{ |I_{\pm \mu_2}|}{n} \cdot C_1 \sigma \sqrt d \r]\\
    &\overset{(iii)}\geq \alpha |a_j| \cdot \l[ \f{\exp(-2)}{32} - \f 32 \cdot \f{ |\calN|}{n}  - \f{ C_1}{C} \r]\\
    &\overset{(iv)}\geq \alpha |a_j| \l[ \f{ \exp(-2)}{32} -  \f 3 2 \l( \f 1 C + C_1\sqrt{\f{1}{C}} \r) - \f{ C_1}{C} \r] \\
    &\overset{(v)}\geq  \f{\exp(-2)\alpha |a_j|}{64} .\numberthis \label{eq:margin.increment.node.j}
\end{align*}
In $(i)$ we use Lemma~\ref{lemma:data.initialization}: for the sums over $\imct {+\mu_1}{\tau}$ and $\imnt{+\mu_1}{\tau}$, part (b) of the lemma and the assumption on $\sigma$ given in Assumption~\ref{a:sigma} imply that $\snorm{x_i-\mu_1}\leq 1/2$ for $C$ large enough and hence $\sip{x_i}{\mu_1} = \sip{x_i-\mu_1}{\mu_1}+1 \in [1/2, 3/2]$ for $i\in I_{+\mu_1}$.  For the sums over $i\in I_{\pm \mu_2}$, part (a) of Lemma~\ref{lemma:data.initialization} implies $|\sip{x_i}{\mu_1}|\leq C_1 \sigma \sqrt d$ and using this with $|y_i \phi'| \leq 1$ provides the desired bound.  
For inequality $(ii)$, we use~\eqref{eq:loss.lb} and that $\ell$ is 1-Lipschitz. 
In inequality $(iii)$, we use the lower bound $|\imct {+\mu_1}{\tau}| \geq n/8$ given in Eq.~\eqref{eq:Icleanmu1.lb.alltimes}, as well as $|\imt {\pm \mu_2}{\tau}|\leq n$ and the upper bound for $\sigma$ given in Assumption~\ref{a:sigma}.  For the inequality $(iv)$,  we use Lemma~\ref{lemma:data.initialization}(c) and the assumptions on the noise rate and number of samples given in  Assumptions~\ref{a:noisy.fraction} and~\ref{a:samples} to bound $|\calN|/n \leq \eta + C_1 \sqrt{\log(1/\delta)/n} \leq 1/C + C_1\sqrt{1/C}$.  Then $(v)$ follows by taking $C$ to be a large enough universal constant. 

Summing~\eqref{eq:margin.increment.node.j} from $\tau = 1, \dots, T-1$ and using that $j\in J_{+\mu_1}$ implies $\sip{\wt 1_j}{\mu_1}>0$, we get that
\begin{equation} \label{eq:margin.time.t.node.j}
\sip{\wt {T}_j}{\mu_1} \geq\sip{\wt {T}_j - \wt 1_j}{\mu_1} \geq \f{ \exp(-2)\alpha |a_j| (T-1) }{64} ,\quad \text{for all } j\in J_{+\mu_1}.
\end{equation}
Thus, we have the following lower bound on the network output at $\mu_1$:
\begin{align*}
    f^J(\mu_1; \Wt {T}) &= \sum_{j\in J_{+\mu_1}} a_j \phi(\sip{\wt  {T}_j}{\mu_1}) +  \sum_{j\in J_{-\mu_1}} a_j \phi(\sip{\wt  {T}_j}{\mu_1}) +  \sum_{j\in J_{\pm \mu_2}} a_j \phi(\sip{\wt  {T}_j}{\mu_1})\\
     &\overset{(i)}= \sum_{j\in J_{+\mu_1}} a_j \sip{\wt  {T}_j}{\mu_1} +  \sum_{j\in J_{\pm \mu_2}} a_j \phi(\sip{\wt  {T}_j}{\mu_1})\\
     &\overset{(ii)}\geq \sum_{j\in J_{+\mu_1}} a_j \sip{\wt  {T}_j}{\mu_1} -  \sum_{j\in J_{\pm \mu_2}} |a_j \sip{\wt  {T}_j}{\mu_1}|\\
     &\overset{(iii)}\geq \sum_{j\in J_{+\mu_1}} a_j \sip{\wt  {T}_j}{\mu_1} -  \frac{3 \alpha |J_{\pm \mu_2}|}{m} \\
     %&\overset{(iv)}\geq a_j \cdot |J_{+\mu_1}| \cdot \f{\exp(-2) \alpha |a_j| t_1}{64} - \frac{\cao  \alpha  |J_{\pm \mu_2}|}{m} \\
     &\overset{(iv)}{\ge} \alpha \l[ \f{ |J_{+\mu_1}| (T-1)\exp(-2)}{64m} - \f{ 3  |J_{\pm \mu_2}|}{m} \r] \\
     &\overset{(v)}\geq \alpha \l[ \f{ (T-1) \exp(-2) }{256 }(1-1/C_0)^2 - 3  \r].
\end{align*}
In $(i)$ we use the neuron alignment condition.  In $(ii)$ we use that $\phi$ is 1-Lipschitz. In $(iii)$ we use the almost-orthogonality (Condition~\ref{condition:almost.orthogonal}) and that $|a_j| = 1/\sqrt{m}$. In~$(iv)$ we use Eq.~\eqref{eq:margin.time.t.node.j} and again use the fact that $|a_j| = 1/\sqrt{m}$. Finally, $(v)$ uses Lemma~\ref{lemma:candidate.subnetwork}, so that we have $|J_{+\mu_1}|/ m \geq \f 14 (1-1/C_0)^2$, as well as the fact that $|J_{\pm \mu_2}| \le m$. 
In particular, we see that for $T-1= 1 /(4\alpha)$, we have
\begin{equation} \label{eq:subnetwork.margin.lb.intermediate}
f^J(\mu_1; \Wt {T}) \geq \f{\exp(-2)}{1024 }(1 - 1/C_0)^2 - 3\alpha   \geq \f{\exp(-2)}{2048 }(1-1/C_0)^2.
\end{equation}
In the last inequality, we use the Assumption~\ref{a:stepsize} and take $C>1$ large enough so that $\alpha \leq \exp(-2)(1-1/C_0)^2/(6 \cdot 1024)$.  With a lower bound on the margin for the cluster center $\mu_1$ established, we can translate this result to one for samples using Lemma~\ref{lemma:data.initialization}.  To do so, note that the sub-network $f^J(\cdot; W)$ is $\snorm{W}_F$-Lipschitz in the network input, i.e., we have
\begin{align*} 
|f^J(x; W) - f^J(x'; W)| &= \l| \sum_{j\in J} a_j [\sigma(\sip {w_j}x) - \sigma(\sip {w_j}{x'}) ]\r| \\
&\leq \snorm{a} \sqrt{\summ j m \sip{w_j}{x-x'}^2} \\
&\leq \snorm{W}_F \snorm{x-x'},
\end{align*}
where the first inequality follows by Cauchy--Schwarz inequality and the last inequality follows since $\lv a \rv = \sum_{j=1}^m a_j^2 = 1$ and $\snorm{W(x-x')}\leq \snorm{W}_F\snorm{x-x'}$.  
Therefore we can use Lemma~\ref{lemma:data.initialization}~(b) to translate~\eqref{eq:subnetwork.margin.lb.intermediate} into a guarantee for the samples.  For any $i\in \imct{+\mu_1}{t}$, so that $\yit i{t}=+1$,
\begin{align*} 
\yit i{t}  f^J(\xit i{t}; \Wt {T}) &\geq \yit i{t} f^J(\mu_1; \Wt {T}) - \snorm{\Wt {T}}_F\max_i\snorm{\xit i{t}-\mu_1}  \\
&\geq \f{ \exp(-2)}{2048}(1-1/C_0)^2 - C_1 \sigma \sqrt d \\
&\geq \f{\exp(-2)}{4096}(1-1/C_0)^2.
\end{align*} 
The second inequality uses that $\snorm{\Wt {T}}_F\leq 1$ and Lemma~\ref{lemma:data.initialization}, while the last inequality uses Assumption~\ref{a:sigma} so that $C_1 \sigma \sqrt d$ can be taken smaller than any absolute constant for $C>1$ sufficiently large. 

This completes the proof for samples $i\in \imct {+\mu_1}{t}$.  To see that the network also incorrectly classifies noisy samples, take $i\in \imnt{+\mu_1}{t}$, so that $\yit i{t}=-1$.  Then, again using Lemma~\ref{lemma:data.initialization}(b),
\begin{align*} 
\yit i{T}  f^J(\xit i{t} \Wt {T}) &= -f^J(\xit i{t} \Wt {T}) \\
&\leq -f^J(\mu_1; \Wt {T}) + \snorm{\Wt {T}}_F\max_i\snorm{\xit i{t}-\mu_1}  \\
&\leq -\f{\exp(-2)}{4096}(1-1/C_0)^2,
\end{align*}
where the last inequality follows since $\lv W^{(T)}_F\rv\le 1$ as we proved above.

For the other clusters, an identical argument to~\eqref{eq:margin.increment.node.j} yields
\begin{align*}
    \sip{\wt {\tau+1}_j - \wt \tau_j}{-\mu_1} &\geq \f{ \alpha |a_j|}{64} \exp(-2),\quad \text{for all } j\in J_{-\mu_1}, \tau \leq T-1,\\
    \sip{\wt {\tau+1}_j - \wt \tau_j}{\mu_2} &\geq \f{ \alpha |a_j|}{64} \exp(-2),\quad \text{for all } j\in J_{+\mu_2}, \tau \leq T-1,\\
    \sip{\wt {\tau+1}_j - \wt \tau_j}{-\mu_2} &\geq \f{ \alpha |a_j|}{64} \exp(-2),\quad \text{for all } j\in J_{-\mu_2}, \tau \leq T-1.\numberthis \label{eq:margin.time.t.other.nodes}
\end{align*}
We can utilize the identities~\eqref{eq:margin.time.t.other.nodes} and similar arguments to show that the desired margin condition holds for other clusters $\imct {-\mu_1}{t}, \imct {\pm \mu_2}{t}$ so the result holds for all $i\in \calct t$. 
\end{proof} 

\subsection{Gradient Descent Produces a Large Margin Classifier}\label{sec:gd.produces.good.subnetwork}
In this section, we show that the sufficient conditions necessary for producing a good subnetwork described in Lemma~\ref{lemma:subnetwork.margin.time.t1} hold.  The first step for this is to show that neuron alignment holds at time $t=1$.
\subsubsection{Proof of Lemma~\ref{lemma:nac.holds.time.1}}
We restate and prove Lemma~\ref{lemma:nac.holds.time.1} below.
\nacholds*
\begin{proof}
Since a good run occurs, all of the events in Lemma~\ref{lemma:random.init}, Lemma~\ref{lemma:candidate.subnetwork},  Lemma~\ref{lemma:data.initialization}, and Lemma~\ref{lemma:Nj+.argument.all.times} hold. 
Recall that the sets $J_{\pm \mu_1}$ and $J_{\pm \mu_2}$ were defined in Lemma~\ref{lemma:candidate.subnetwork}.   
We will now show that Condition~\ref{condition:alignment.condition} holds for these sets at time $t=1$. We will demonstrate the first claim in the condition statement (regarding $\mu_1$), that is,
for all $j \in J_{+\mu_1}$:
\begin{align*}
        \phi'(\sip{\wt {1}_j} {\xit k1}) &= 1\quad \text{for all } k\in \imt {+\mu_1}{0},\\ \phi'(\sip{\wt {1}_j} {\xit k1})& = 0\quad  \text{for all } k\in \imt {-\mu_1}{0}.
\end{align*}
The remaining parts of the neuron alignment condition concerning $j\in J_{-\mu_1}\cup J_{\pm \mu_2}$ shall follow by using an identical argument.  

There are two parts to the neuron alignment condition, let us begin by proving that the first part holds.

\paragraph{Part 1 of NAC:} Let us begin by showing that for all $j \in J_{+\mu_1}$:
\begin{align}\label{eq:t1.J1+.identity.noisy}
    \phi'(\sip{\wt {1}_j} {\xit k1})= 1\quad \text{for all } k\in \imt {+\mu_1}{0}.
\end{align}
Recall that by the definition of the set $J_{+\mu_1}$, we have that for all $j \in J_{+\mu_1}$,
\begin{align*}
    \phi'(\langle w_{j}^{(0)},\mu_1\rangle)=1.
\end{align*}
To show that the first part of NAC holds for the subset $J_{+\mu_1}$, we need to show that a step of gradient descent takes ensures that all of the samples from this cluster are captured by the neurons in $J_{+\mu_1}$. We shall prove this in stages. 
\begin{enumerate}
    \item First, we shall establish a relation between the parameters after one the first step of gradient descent $w_j^{(1)}$ and those at initialization $w_{j}^{(0)}$. 
    \item Then, we shall leverage this relation  to show that the angle between $w_j^{(1)}$ and $\mu_1$ is small.
    \item This, along with the fact that the samples from this cluster are close to its center, shall be sufficient to ensure that \eqref{eq:t1.J1+.identity.noisy} is satisfied.
\end{enumerate}
 
\paragraph{Step 1:} First, recall that by the calculation~\eqref{eq:nn.output.bound.spectral.norm}, we have $|f(x_i; \Wt 0)| \leq \snorm{\Wt 0}_F \snorm{x_i}$.  Thus the bound $\lv w_{j}^{(0)}\rv \le \frac{3}{2}\sinit \sqrt{d}$ by Lemma~\ref{lemma:random.init}, the bound $\lv x_i\rv \le \sqrt{2}$ from~\eqref{eq:xi.norm.bound} imply that %and because $|a_j| = 1/\sqrt{m}$ imply that
\[ |f(x_i; \Wt 0)| \leq \snorm{\Wt 0}_F \snorm{x_i} \leq 3 \sinit \sqrt{md} .\]
Note that $z \mapsto -\ell'(z)$ is a decreasing function, and thus
\begin{align}\label{e:all_loss_concentrates.noisy}
     -\ell'_{i,0} & \in \left[-\ell'\left(3\sinit \sqrt{md} \right),-\ell'\left(-3\sinit \sqrt{md}\right)\right].
\end{align}
With this in place, let us analyze the gradient update for a neuron in the set $J_{+\mu_1}$. Recall that for such nodes, $a_j=1/\sqrt{m}>0$ and therefore,
\begin{align*}
    &w_{j}^{(1)}\\ & = w_{j}^{(0)}+\frac{\alpha a_j}{n}\sum_{i=1}^n -\ell'_{i,0}y_ix_i \phi'(\langle w_j^{(0)},x_i\rangle) \\
        & = w_{j}^{(0)}+\frac{\alpha a_j}{n}\sum_{i\in \imt {+\mu_1}{0}} -\ell'_{i,0} y_i \phi'(\langle w_j^{(0)},x_i\rangle) \mu_{1} +\frac{\alpha a_j}{n}\sum_{i\in \imt {-\mu_1}{0}} -\ell'_{i,0} y_i\phi'(\langle w_j^{(0)},x_i\rangle)(-\mu_{1}) \\
    &\qquad +\frac{\alpha a_j}{n}\sum_{i\in \imt {+\mu_2}{0}} -\ell'_{i,0}y_i \phi'(\langle w_j^{(0)},x_i\rangle) \mu_2 +\frac{\alpha a_j}{n}\sum_{i\in \imt {-\mu_2}{0}} -\ell'_{i,0}y_i \phi'(\langle w_j^{(0)},x_i\rangle) (-\mu_2)\\
    &\qquad +\frac{\alpha a_j}{n}\sum_{i\in \imt {+\mu_1}{0}} -\ell'_{i,0}y_i \phi'(\langle w_j^{(0)},x_i\rangle) (x_i-\mu_{1}) +\frac{\alpha a_j}{n}\sum_{i\in \imt {-\mu_1}{0}} -\ell'_{i,0}y_i \phi'(\langle w_j^{(0)},x_i\rangle) (x_i-(-\mu_1))\\
    &\qquad +\frac{\alpha a_j}{n}\sum_{i\in \imt {+\mu_2}{0}} -\ell'_{i,0} y_i\phi'(\langle w_j^{(0)},x_i\rangle) (x_i-\mu_2) +\frac{\alpha a_j}{n}\sum_{i\in \imt {-\mu_2}{0}} -\ell'_{i,0} y_i\phi'(\langle w_j^{(0)},x_i\rangle)(x_i-(-\mu_2)).
\end{align*}
Define the first ``error vector''
\begin{align*}
    \zeta_1 &:=\frac{\alpha a_j}{n}\sum_{i\in \imt {+\mu_1}{0}} -\ell'_{i,0}y_i \phi'(\langle w_j^{(0)},x_i\rangle) (x_i-\mu_{1}) +\frac{\alpha a_j}{n}\sum_{i\in \imt {-\mu_1}{0}} -\ell'_{i,0} y_i\phi'(\langle w_j^{(0)},x_i\rangle) (x_i-(-\mu_1))\\
    &\qquad +\frac{\alpha a_j}{n}\sum_{i\in \imt {+\mu_2}{0}} -\ell'_{i,0}y_i \phi'(\langle w_j^{(0)},x_i\rangle) (x_i-\mu_2) +\frac{\alpha a_j}{n}\sum_{i\in \imt {-\mu_2}{0}} -\ell'_{i,0} y_i \phi'(\langle w_j^{(0)},x_i\rangle)(x_i-(-\mu_2)). 
\end{align*}
By Lemma~\ref{lemma:data.initialization}(b) that provides a bound on the deviation of  $x_i$ from its cluster center, and using that $|\ell'(t)|,|\phi'(t)|\leq 1$, we have that
\begin{align}\label{e:zeta_1_bound_new.noisy}
    \snorm{\zeta_1} &\leq C_1 \alpha a_j \sigma  \sqrt d.
\end{align}
    
Continuing from above, we get,
\begin{align*}
   & w_{j}^{(1)} - w_{j}^{(0)} \\ &\qquad  = \frac{\alpha a_j}{n}\sum_{i\in \imt {+\mu_1}{0}} -\ell'_{i,0} y_i \phi'(\langle w_j^{(0)},x_i\rangle) \mu_{1} -\frac{\alpha a_j}{n}\sum_{i\in \imt {-\mu_1}{0}} -\ell'_{i,0} y_i \phi'(\langle w_j^{(0)},x_i\rangle) \mu_1\\
    &\quad \qquad +\frac{\alpha a_j}{n}\sum_{i\in \imt {+\mu_2}{0}} -\ell'_{i,0} y_i\phi'(\langle w_j^{(0)},x_i\rangle) \mu_2 - \frac{\alpha a_j}{n}\sum_{i\in \imt {-\mu_2}{0}} -\ell'_{i,0} y_i \phi'(\langle w_j^{(0)},x_i\rangle)\mu_2 +\zeta_1\\
    &\qquad = \frac{\alpha a_j}{n}\sum_{i\in \imt {+\mu_1}{0}}-\ell'(0)y_i \phi'(\langle w_j^{(0)},x_i\rangle) \mu_{1} -\frac{\alpha a_j}{n}\sum_{i\in \imt {-\mu_1}{0}}-\ell'(0) y_i \phi'(\langle w_j^{(0)},x_i\rangle) \mu_1\\
    &\quad \qquad + \frac{\alpha a_j}{n}\sum_{i\in \imt {+\mu_2}{0}} -\ell'(0)y_i\phi'(\langle w_j^{(0)},x_i\rangle) \mu_2  -\frac{\alpha a_j}{n}\sum_{i\in \imt {-\mu_2}{0}} -\ell'(0) y_i \phi'(\langle w_j^{(0)},x_i\rangle)\mu_2 +\zeta_1 + \zeta_2,\numberthis \label{eq:wj1.wj0.difference.z1.z2}
\end{align*}
where we have defined the second ``error vector'' $\zeta_2$ as,
\begin{align*}
    \zeta_2 &:= \frac{\alpha a_j}{n}\sum_{i\in \imt {+\mu_1}{0}} (-\ell'_{i,0}+\ell'(0)) \phi'(\langle w_j^{(0)},x_i\rangle)y_i \mu_{1} -\frac{\alpha a_j}{n}\sum_{i\in \imt {-\mu_1}{0}} (-\ell'_{i,0}+\ell'(0))y_i \phi'(\langle w_j^{(0)},x_i\rangle)\mu_1 \\
    &\qquad -\frac{\alpha a_j}{n}\sum_{i\in \imt {+\mu_2}{0}} (-\ell'_{i,0}+\ell'(0))y_i \phi'(\langle w_j^{(0)},x_i\rangle) \mu_2 +\frac{\alpha a_j}{n}\sum_{i\in \imt {-\mu_2}{0}} (-\ell'_{i,0}+\ell'(0))y_i \phi'(\langle w_j^{(0)},x_i\rangle)\mu_2.
\end{align*}
Applying the triangle inequality and Equation~\eqref{e:all_loss_concentrates.noisy},
\begin{align}
    \nonumber \|\zeta_2 \|& \nonumber\le \alpha a_j\max\{\|\mu_1\|,\|\mu_2\|\} \max\left\{\left(-\ell'\left(-3\sinit \sqrt{md}\right)+\ell'(0)\right),\left(-\ell'\left(3\sinit \sqrt{md} \right)+\ell'(0)\right)\right\} \\
    %&\le \alpha a_j\left(\exp\left(2\sinit \lVert x_i\rVert \sqrt{\log(1/\delta)}\right)-1\right) \\
    &\leq \alpha a_j \sinit \sqrt{md},\numberthis \label{e:zeta_2_bound_new.noisy}
\end{align}
where in the last line we have used that $-\ell'$ is $\nicefrac 14$-Lipschitz and that $\lv \mu_1 \rv = \lv \mu_2\rv=1$.  
Define now, for $j\subset [m]$,
\begin{equation} 
\begin{cases} 
&N_{+\mu_1}(j)= \sum_{i\in \imt {+\mu_1}{0}} y_i \phi'(\sip {w_j^{(0)}}{x_i}),\\
&N_{-\mu_1}(j)= \sum_{i\in \imt {-\mu_1}{0}} y_i \phi'(\sip{w_j^{(0)}}{x_i}),\\
&N_{+\mu_2}(j)= \sum_{i\in \imt {+\mu_2}{0}} y_i \phi'(\sip{w_j^{(0)}}{x_i}),\\
&N_{-\mu_2}(j)= \sum_{i\in \imt {-\mu_2}{0}} y_i \phi'(\sip{w_j^{(0)}}{x_i}).\\
\end{cases}\end{equation}
Substituting the above definition into~\eqref{eq:wj1.wj0.difference.z1.z2}, we then have
\begin{align}
    \nonumber &w_{j}^{(1)} - w_{j}^{(0)}\\  \nonumber &\qquad  = \frac{\alpha a_j}{n}\left[ -\ell'(0)\mu_1(N_{+\mu_1}(j)-N_{-\mu_1}(j))-\ell'(0)\mu_2(N_{+\mu_2}(j)-N_{-\mu_2}(j))\right]+\zeta_1 +\zeta_2\\
    &\qquad  = \frac{\alpha a_j}{2n}\left[ \mu_1(N_{+\mu_1}(j)-N_{-\mu_1}(j))+\mu_2(N_{+\mu_2}(j)-N_{-\mu_2}(j))\right]+\zeta_1 +\zeta_2,
    \label{eq:wj1.vs.wj0.identity}
\end{align}
where the last equality follows since $-\ell'(0)=1/2$.

\paragraph{Step 2:} Continuing with the plan outlined above, we will now show that $\sip{\wt 1_j / \snorm{\wt 1_j}}{\mu_1}\geq c$ for a universal constant $c$. 
We have, 
\begin{align*}
    \langle w_j^{(1)} - \wt 0_j,\mu_1\rangle &= \frac{\alpha a_j}{2n}\left[ \lVert \mu_1 \rVert^2 (N_{+\mu_1}(j)-N_{-\mu_1}(j))+\langle \mu_2,\mu_1\rangle(N_{+\mu_2}(j)-N_{-\mu_2}(j))\right]\\
    & \qquad    +\sip{\zeta_1}{\mu_1} +\sip{\zeta_2}{\mu_1} \\
     &\geq \frac{\alpha a_j}{2n}\left[  N_{+\mu_1}(j)-N_{-\mu_1}(j)\right]  - \frac{\alpha a_j}{n}\l[C_1 n \sigma \sqrt d + n \sinit \sqrt{md} \r]. \numberthis \label{e:margin_lower_bound_step_one_new.noisy}
\end{align*}
In the last line we have applied the inequalities~\eqref{e:zeta_1_bound_new.noisy} and \eqref{e:zeta_2_bound_new.noisy}. 
Thus, it suffices to derive a lower bound for $N_{+\mu_1}(j)-N_{-\mu_1}(j)$, which is precisely the result that Lemma~\ref{lemma:Nj+.argument.all.times} provides.   We have,
\begin{align}\nonumber 
   & N_{+\mu_1}(j) - N_{-\mu_1}(j)\\ &\qquad = \sum_{i\in \imt {+\mu_1}{0}} y_i \phi'(\sip{\wt 0_j}{x_i}) - \sum_{i\in \imt {-\mu_1}{0}} \phi'(\sip{\wt 0_j}{x_i}) \\ \nonumber
    &\qquad = \sum_{i\in \imct {+\mu_1}{0}} \phi'(\sip{\wt 0_j}{x_i}) - \sum_{i\in \imnt {+\mu_1}{0}} \phi'(\sip{\wt 0_j}{x_i})- \sum_{i\in \imct {-\mu_1}{0}} \phi'(\sip{\wt 0_j}{x_i}) +\sum_{i\in \imnt {-\mu_1}{0}} \phi'(\sip{\wt 0_j}{x_i}) \\ \nonumber 
    &\qquad \geq \sum_{i\in \imct {+\mu_1}{0}} \phi'(\sip{\wt 0_j}{x_i})- \sum_{i\in \imct {-\mu_1}{0}} \phi'(\sip{\wt 0_j}{x_i}) - |\calnt 0|\\ \nonumber 
    &\qquad \overset{(i)}\geq n \l( \f{1}{C_2} - \f{ |\calnt 0|}{n} \r) \\
    &\qquad \overset{(ii)}\geq \f{n}{2C_2}.
\end{align} 
In $(i)$ we use Lemma~\ref{lemma:Nj+.argument.all.times}, while in $(ii)$ we use Assumption~\ref{a:noisy.fraction} so that $|\calnt 0|/n \leq 2 \eta \leq 1/2C_2$.  
Thus, plugging this in to~\eqref{e:margin_lower_bound_step_one_new.noisy}  we get that
\begin{align}%\nonumber 
    \sip{\wt 1_j}{\mu_1} &\overset{(i)}> \sip{\wt 1_j - \wt 0_j}{\mu_1} \geq \frac{\alpha a_j}{4C_2}   -\alpha a_j\l[C_1 \sigma  \sqrt{d} + \sinit \sqrt{md} \r] \overset{(ii)}\geq \alpha a_j  / 8 C_2. \label{eq:unnormalized.margin}
\end{align}
Inequality $(i)$ uses that $j\in J_{+\mu_1}$ implies $\sip{\wt 0_j}{\mu_1}>0$.  Inequality $(ii)$ follows by using Assumption~\ref{a:sigma}, so that $C_1 \sigma \sqrt{d}\leq 1/16C_2$, as well as Assumption~\ref{a:sinit} so that for $C>1$ sufficiently large we have $\sinit \sqrt{md}\leq 1/16C_2$.  
Next, we can use Lemma~\ref{lemma:per.neuron.norm} to derive a bound for the normalized margin,
\begin{equation}\label{eq:neuron.normalized.margin.t1}
    \ip{\f{ \wt 1_j }{ \snorm{\wt 1_j}} }{\mu_1} \geq \f{\alpha a_j / 8 C_2}{2 \alpha a_j} = \f{1}{16 C_2}.
\end{equation}
This completes the proof for the normalized margin claim. 

\paragraph{Step~3:}
To show that the first part of the neuron alignment holds, we want to show that $\sip{\wt 1_j}{\xit i1}>0$.  We have,
\begin{align*}
    \ip{\frac{\wt 1_j}{\snorm{\wt 1_j}}}{\xit i1} &= \ip{\frac{\wt 1_j }{ \snorm{\wt 1_j}}}{\mu_1} + \ip{\frac{\wt 1_j}{\snorm{\wt 1_j}}}{\xit i1-\mu_1}\\
    &\overset{(i)}\geq 1/16C_2  -\snorm{\xit i1 - \mu_1}  \\
    &\overset{(ii)}\geq 1/16C_2 - C_1 \sigma \sqrt d \\
    &\overset{(iii)} \geq 1/32 C_2 > 0. \numberthis \label{eq:correlation.with.cluster.mean.to.samples}
\end{align*}
Above, $(i)$ uses~\eqref{eq:neuron.normalized.margin.t1} and the Cauchy--Schwarz inequality.  Inequality $(ii)$ uses Lemma~\ref{lemma:data.initialization}.  The final inequality $(iii)$ uses Assumption~\ref{a:sigma}, so that $C_1 \sigma \sqrt d \leq 1/ 64 C_2$.   This completes the part of neuron alignment concerning neurons $J_{+\mu_1}$ and for samples in cluster $\imt {+\mu_1}{1}$.  

\paragraph{Part 2 of NAC:} To show the part of neuron alignment concerning samples in cluster $\imt {-\mu_1}{1}$, note that we still have the identity~\eqref{eq:neuron.normalized.margin.t1}.  But for samples $i\in \imt {-\mu_1}{1}$, we have
\[ 
\sip{\wt 1_j}{\xit i1} = \sip{\wt 1_j}{-\mu_1} + \sip{\wt 1_j}{\xit i1+ \mu_1},\]
where $\snorm{\xit i1+\mu_1}$ is small, and so the inequality $\sip{\wt 1_j}{\xit i1}<0$ follows using the same argument as above. Hence, we have shown that $\phi'(\sip{\wt 1_j}{\xit i1})$ for all $i \in I_{-\mu_1}$.

This completes the proof of neuron alignment for the neurons in $J_{+\mu_1}$. An analogous argument can also be used to establish the claim for the neurons in $J_{-\mu_1} \cup J_{\pm \mu_2}$.
\end{proof}

\subsubsection{Proof of Lemma~\ref{lemma:nac.orthogonality.holds.all.times}}
We now show that the neuron alignment condition and almost-orthogonality condition hold for a sufficiently large amount of time.

\nacaohold*

\begin{proof}
%Throughout this proof, we will denote $(x_i, y_i) = (x_i, y_i)$ for notational simplicity.  
%We will show by induction that for every $t=1, \dots, 1/(4\alpha)$, neuron alignment holds at time $t$ and almost-orthogonality holds up to time $t$.  
The proof is by induction.
To see the base case $t=1$, first, note that neuron alignment holds at time $t=1$ by Lemma~\ref{lemma:nac.holds.time.1}.  Further, almost-orthogonality holds at time $t=1$ since by Lemma~\ref{lemma:per.neuron.norm} we have $|\sip{\wt 1_j}{\mu}|\leq \snorm{\wt 1_j} \leq 2 |a_j| \alpha t$ for any $\mu\in \{\pm \mu_1, \pm \mu_2\}$.   So let us now assume that neuron alignment and almost-orthogonality hold at every time step until time $t$, and consider the case $t+1\leq 1/(4\alpha)$.  
%The base case regarding neuron alignment follows by Lemma~\ref{lemma:nac.holds.time.1}.  The base case for almost-orthogonality follows by Lemma~\ref{lemma:per.neuron.norm}, since the lemma shows that $|\sip{\wt 1_j}{\mu}|\leq \snorm{\wt 1_j} \leq 2 |a_j| \alpha t$ for any $\mu\in \{\pm \mu_1, \pm \mu_2\}$. 
%So let us suppose the result holds up to time $t$, and consider time $t+1\leq 1/(4\alpha)$. 
By Lemma~\ref{lemma:per.neuron.norm}, since $t+1\leq 1/(4\alpha)$, we have $\snorm{\Wt \tau}_F \leq 1$ for every $\tau \leq t+1$.  Using an identical argument to~\eqref{eq:loss.lb}, this implies for all $i \in [n]$ and $\tau \le \left\{1,\ldots,1/(4\alpha)\right\}$,
\begin{equation}\label{eq:loss.lb.induction}
 -\ell'_{i,\tau} := -\ell'(\yit i\tau f(\xit i\tau; \Wt \tau)) \geq \frac 12 \exp(-2).
\end{equation}
This key property will allow us to show that neuron alignment holds at time $t+1$. % We prove neuron alignment in parts. 

\paragraph{Neuron alignment holds at time $t+1$.}  We will first show the result for neurons $j\in J_{+\mu_1}$; the result for neurons in $J_{-\mu_1}\cup J_{\pm \mu_2}$ will follow similarly.   

Let $j\in J_{+\mu_1}$, so $a_j = |a_j| = 1/\sqrt{m}$.  It suffices to show that for $k\in \imt {+\mu_1}t$, we have $\sip{\wt {t+1}_j}{\xit k{t+1}} >0$, and for $k\in \imt {-\mu_1}t$, we have $\sip{\wt {t+1}_j}{\xit k{t+1}} < 0$.   To show this, we will utilize an argument similar to that we used in the proof of Lemma~\ref{lemma:nac.holds.time.1} (see eqs.~\eqref{eq:neuron.normalized.margin.t1} and~\eqref{eq:correlation.with.cluster.mean.to.samples}), in that we will first show that $\sip{\wt {t+1}_j / \snorm{\wt {t+1}_j}}{+\mu_1}\geq c$ for some constant $c>0$, and then use that the within-cluster variance is of order $\sigma^2 d$ and that $\sigma^2 \ll 1/d$. Towards this end, we first derive a consequence of neuron alignment.  Let $\tau$ be a time satisfying $1 \leq \tau \leq t$.  Then neuron alignment holds at time $\tau$ by the induction hypothesis, so that,
\begin{align*}\nonumber 
   &\sip{ w_{j}^{(\tau+1)} - w_{j}^{(\tau)}}{+\mu_1} \\
   &=  \frac{\alpha a_j}{n}\sum_{i=1}^n -\ell'_{i,\tau} \phi'(\langle w_j^{(\tau)},\xit i\tau\rangle) \sip{\yit i\tau \xit i\tau }{\mu_1} \\ \nonumber 
    & = \frac{\alpha |a_j|}{n}\sum_{i\in \imct {+\mu_1}{\tau}} -\ell'_{i,\tau} \phi'(\langle w_j^{(\tau)},\xit i\tau\rangle) \sip{\xit i\tau}{\mu_1} - \frac{\alpha |a_j|}{n}\sum_{i\in \imnt {+\mu_1}{\tau}} -\ell'_{i,\tau} \phi'(\langle w_j^{(\tau)},\xit i\tau\rangle)\sip{\xit i\tau}{\mu_1} \\ \nonumber 
    &\qquad +  \frac{\alpha |a_j|}{n}\sum_{i\in \imct {-\mu_1}{\tau}} -\ell'_{i,\tau} \phi'(\langle w_j^{(\tau)},\xit i\tau\rangle) \sip{\xit i\tau}{\mu_1} -\frac{\alpha |a_j|}{n}\sum_{i\in \imnt {-\mu_1}{\tau}} -\ell'_{i,\tau} \phi'(\langle w_j^{(\tau)},\xit i\tau\rangle)\sip{\xit i\tau}{\mu_1} \\ \nonumber 
    &\qquad +\frac{\alpha |a_j|}{n}\sum_{i\in \imt {+\mu_2}{\tau}} -\ell'_{i,\tau} \yit i\tau \phi'(\langle w_j^{(\tau)},\xit i\tau\rangle) \sip{\xit i\tau}{\mu_1} +\frac{\alpha |a_j|}{n}\sum_{i\in \imt {-\mu_2}{\tau}} -\ell'_{i,\tau} \yit i\tau \phi'(\langle w_j^{(\tau)},\xit i\tau\rangle)\sip{\xit i\tau}{\mu_1} \\ \nonumber 
     & \overset{(i)}= \frac{\alpha |a_j|}{n}\sum_{i\in \imct {+\mu_1}{\tau}} -\ell'_{i,\tau}\sip{\xit i\tau}{\mu_1} - \frac{\alpha |a_j|}{n}\sum_{i\in \imnt {+\mu_1}{\tau}} -\ell'_{i,\tau}\sip{\xit i\tau}{\mu_1} \\
    &\qquad +\frac{\alpha |a_j|}{n}\sum_{i\in \imt {+ \mu_2}{\tau}} -\ell'_{i,\tau} \yit i\tau \phi'(\langle w_j^{(\tau)},\xit i\tau\rangle) \sip{\xit i\tau}{\mu_1} +\frac{\alpha |a_j|}{n}\sum_{i\in \imt {- \mu_2}{\tau}} -\ell'_{i,\tau}\yit i\tau \phi'(\langle w_j^{(\tau)},\xit i\tau\rangle)\sip{\xit i\tau}{\mu_1}.\numberthis \label{eq:J1+.activation.identity.noisy}
\end{align*}
In $(i)$ we have used that the neuron alignment condition holds at time $\tau$, and thus $\phi'(\sip{\wt \tau_j}{\xit i\tau})=1$ for $i\in \imt {+\mu_1}\tau$ and $\phi'(\sip{\wt \tau_j}{\xit i\tau})=0$ for $i\in \imt {-\mu_1}\tau$.  
We can bound the terms $\sip{\xit i\tau}{+\mu_1}$ appearing above with Lemma~\ref{lemma:data.initialization}, so that
\begin{align*}
    \sip{ w_{j}^{(\tau+1)} - w_{j}^{(\tau)}}{+\mu_1} &\overset{(i)}\geq \f {\alpha |a_j|}{n} \Bigg[ \sum_{i\in \imct {+\mu_1}{\tau}} -\ell'_{i,\tau} \l[1 - C_1 \sigma \sqrt{d}\r] -  2|{\calnt \tau}| -2 \sum_{i\in \imt {\pm \mu_2}{\tau}} -\ell'_{i,\tau} C_1 \sigma \sqrt{d}\Bigg]\\
     &\overset{(ii)}\geq\f {\alpha |a_j|}{n} \Bigg[ \sum_{i\in \imct {+\mu_1}{\tau} }\f 1 4 \exp(-2) -  2 |{\calnt \tau}| -2C_1 |\imct {+\mu_1}{\tau} \cup \imt {\pm \mu_2}{\tau}| \sigma \sqrt{d} \Bigg] \\
     &\overset{(iii)}\geq\f {\alpha |a_j|}{n} \Bigg[ \f n 8 \cdot \f 1 4 \exp(-2) -  2 |{\calnt \tau}| -2C_1 n \sigma \sqrt{d} \Bigg] \\
     &\overset{(iv)}\geq \f{ \alpha |a_j|\exp(-2)}{64}.\numberthis \label{eq:neuron.alignment.J+mu1}
\end{align*}
In $(i)$ we use that $|\ell'|\leq 1$ and Lemma~\ref{lemma:data.initialization}, so that $ \sip{\xit i\tau}{\mu_1} \geq 1-C_1 \sigma \sqrt d$ for $i\in \imt {+\mu_1}{\tau}$, $|\sip{x_i}{\mu_1}| \leq 2$ for $i\in \imt {-\mu_1}{\tau}$, and $|\sip{\xit i\tau}{\mu_1}|\leq C_1 \sigma \sqrt d$ for $i\in \imt {\pm \mu_2}{\tau}$.  
In inequality $(ii)$, we use~\eqref{eq:loss.lb.induction} as well as the fact that Assumption~\ref{a:sigma} implies $ C_1 \sigma \sqrt{d}\leq 1/2$.  
In $(iii)$ we use parts (c) and (d) of Lemma~\ref{lemma:data.initialization} and Assumption~\ref{a:samples} so that $|\imct {+\mu_1}{\tau}| \geq |\imt {+\mu_1}{\tau}| - |\calnt \tau| \geq n/8$.   The final line $(iv)$ follows by using Assumptions~\ref{a:sigma} and~\ref{a:noisy.fraction}, so that $2|\calnt \tau|/n \leq \exp(-2)/128$ and $2 C_1 \sigma \sqrt d \leq \exp(-2)/128$ as well.  We have thus shown that if neuron alignment holds at time $\tau$, then for $j\in J_{+\mu}$ we have $\sip{\wt {\tau+1}_j - \wt \tau_j}{+\mu_1}\geq \alpha |a_j|\exp(-2)/64$.  Telescoping this inequality from times $\tau=1, \dots, t$, we get
\[ \sip{\wt {t+1}_j}{+\mu_1} \geq \sip{\wt 1_j}{+\mu_1} + \f{ \alpha |a_j| t \exp(-2)}{64} \overset{(i)}\geq \f{ \alpha |a_j| t \exp(-2)}{64},\]
where inequality $(i)$ uses Lemma~\ref{lemma:nac.holds.time.1}.  By Lemma~\ref{lemma:per.neuron.norm}, we have $\snorm{\wt {t+1}_j}\leq 2 \alpha |a_j|(t+1)$, so that,
\begin{equation}
    \ip{\f{\wt {t+1}_j}{\snorm{\wt {t+1}_j}}}{+\mu_1} \geq \f{ \alpha |a_j| t \exp(-2)}{128 \alpha |a_j| (t+1)} \geq \f{\exp(-2)}{256}.
\end{equation}
Using an identical argument to~\eqref{eq:correlation.with.cluster.mean.to.samples}, since by Lemma~\ref{lemma:data.initialization}(b) and Assumption~\ref{a:sigma} we have the inequalities $\snorm{\xit k{t+1} - \mu_1}\leq  C_1 \sigma \sqrt{d}\leq C_1 / C$, by taking $C>512 C_1 \exp(2)$ we have $\sip{\wt {t+1}_{j}}{\xit k{t+1}} > 0$ for $k\in \imt{+\mu_1}{t+1}$.  A symmetric argument shows that $\sip{\wt {t+1}_j}{\xit k{t+1}} < 0$ for $k\in \imt {-\mu_1}{t+1}$.  
This completes the proof that neuron alignment holds for neurons $j\in J_{+\mu_1}$.  We can show that neuron alignment holds for neurons in $J_{-\mu_1}\cup J_{\pm \mu_2}$ using an analogous argument.
%\paragraph{Neuron alignment holds at time $t+1$: $J_{-\mu_1}\cup J_{\pm \mu_2}$ neurons.}
% To show that neuron alignment holds for neurons $j\in J_{-\mu_1}$, we utilize neuron alignment at time $t$ to derive a decomposition analogous to that of~\eqref{eq:J1+.activation.identity.noisy}.  In particular, for the $J_{+\mu_1}$ neurons we used the fact that $\sip{x_i}{x_k} \geq 1/2$ for $i,k\in \imt {+\mu_1}{t}$ while $|\sip{x_i}{x_k}|\leq C_1 \sigma \sqrt d$ when $i\in \imt{\pm \mu_2}{t}$ and $k\in \imt {+\mu_1}{t}$ and that $-\ell'_{i,t}\in [0.5 \exp(-2), 1]$.  For the $J_{-\mu_1}$ neurons, the bound on the losses $-\ell'_{i,t}\in [0.5\exp(-2), 1]$ still holds and we instead use the fact that $\sip{x_i}{x_k} \geq 1/2$ for $i,k\in \imt {-\mu_1}{t}$ while $|\sip{x_i}{x_k}|\leq C_1 \sigma \sqrt d$ when $i\in \imt{\pm \mu_2}{t}$ and $k\in \imt {-\mu_1}{t}$. 
% For the $J_{\pm \mu_2}$ neurons, the argument follows in an identical manner, with the only difference being that for such neurons, the second layer weights satisfy $a_j= -|a_j|$.  

\paragraph{Almost-orthogonality holds at time $t+1$.} We now show that almost-orthogonality continues to hold at time $t+1$ given it holds at time $t$.  We will prove the result for neurons $j\in J_{+\mu_1}$ with an analogous argument holding for the neurons in $J_{-\mu_1}\cup J_{\pm \mu_2}$.   

We want to show that, for any neuron $j\in J_{+\mu_1}$ satisfying
\[ |\sip{\wt t_j}{\mu_2}| \leq 3 \alpha |a_j|,\]
we have that $|\sip{\wt {t+1}_j}{\mu_2}|\leq 3 \alpha |a_j|$ as well.  We will show this by demonstrating that if at time $t$ we have $|\sip{\wt {t}_j}{\mu_2}| \geq \alpha |a_j|$, then $\sip{\wt {t+1}_j}{\mu_2}$ will either change sign or will decrease in magnitude at the next iteration; since the order of norm changes for a single neuron in one step is $O(\alpha |a_j|)$, this will complete the proof.  

Consider the case that $\sip{\wt {t}_j}{\mu_2} \geq \alpha |a_j|$;
the negative case will follow using a symmetric argument.   Since neuron alignment holds, an identical argument used to derive Equations~\eqref{eq:J1+.activation.identity.noisy} through~\eqref{eq:neuron.alignment.J+mu1} implies that 
\begin{align}\nonumber 
   &\sip{ w_{j}^{(t+1)} - w_{j}^{(t)}}{\mu_2} \\ \nonumber 
   &=  \frac{\alpha |a_j|}{n}\sum_{i\in \imct {+\mu_1}{t}} -\ell'_{i,t}\sip{x_i}{\mu_2} - \frac{\alpha |a_j|}{n}\sum_{i\in \imnt {+\mu_1}{t}} -\ell'_{i,t}\sip{x_i}{\mu_2} \\ \nonumber 
    &\qquad +\frac{\alpha |a_j|}{n}\sum_{i\in \imt {+\mu_2}{t}} -\ell'_{i,t} y_i \phi'(\langle w_j^{(t)},x_i\rangle) \sip{x_i}{\mu_2} +\frac{\alpha |a_j|}{n}\sum_{i\in \imt {- \mu_2}{t}} -\ell'_{i,t} y_i \phi'(\langle w_j^{(t)},x_i\rangle)\sip{x_i}{\mu_2}\\ \nonumber 
    &\overset{(i)}\leq 2 C_1 \alpha |a_j| \sigma \sqrt d - \f{\alpha|a_j|}n \sum_{i\in \imct{+\mu_2}{t}} -\ell'_{i,t} \phi'(\sip{\wt t_j}{x_i}) \sip{x_i}{\mu_2} + \f{\alpha|a_j|}{n} \sum_{i\in \imnt{+\mu_2}{t}} -\ell'_{i,t} \phi'(\sip{\wt t_j}{x_i}) \sip{x_i}{\mu_2} \\ \nonumber 
    &\qquad -\f{\alpha |a_j|}{n} \sum_{i\in \imct{-\mu_2}{t}} -\ell'_{i,t} \phi'(\sip{\wt t_j}{x_i}) \sip{x_i}{\mu_2} + \f{\alpha |a_j|}{n} \sum_{i\in \imnt{-\mu_2}{t}} -\ell'_{i,t} \phi'(\sip{\wt t_j}{x_i}) \sip{x_i}{\mu_2}  \\ \nonumber 
    &\overset{(ii)}\leq 2C_1\alpha |a_j| \sigma \sqrt{d} + \frac{2\alpha |a_j| |{\calnt t}| }{n} \\ \nonumber 
    &\qquad -\frac{\alpha |a_j|}{n}\sum_{i\in \imct {+\mu_2}{t}} -\ell'_{i,t} \phi'(\langle w_j^{(t)},x_i\rangle) \sip{x_i}{\mu_2} - \frac{\alpha |a_j|}{n}\sum_{i\in \imct {-\mu_2}{t}} -\ell'_{i,t} \phi'(\langle w_j^{(t)},x_i\rangle)\sip{x_i}{\mu_2}\\ \nonumber 
    &\overset{(iii)}\leq 2C_1\alpha |a_j| \sigma \sqrt{d} + \frac{2 \alpha |a_j| |{\calnt t}|}{n} \\ \nonumber 
    &\quad -\frac{\alpha |a_j|}{n}\sum_{i\in \imct {+\mu_2}{t}} \f 12 \exp(-2) \phi'(\langle w_j^{(t)},x_i\rangle)\cdot \f 12 +\f 32 \cdot \frac{\alpha |a_j|}{n}\sum_{i\in \imct {-\mu_2}{t}} \phi'(\langle w_j^{(t)},x_i\rangle) \\ \nonumber
    &= \alpha |a_j| \Bigg[  2C_1 \sigma \sqrt{d} + 2 \f{|{\calnt t}|}n \\
    &\quad - \f {\exp(-2)}{4n} \Bigg( \sum_{i\in \imct {+\mu_2}{t}} \phi'(\sip{\wt t_j}{x_i}) - 6\exp(2)\cdot  \sum_{i\in \imct {-\mu_2}{t}} \phi'(\sip{\wt t_j}{x_i}) \Bigg)  \Bigg].\label{eq:almost.orthogonality.mu2.intermediate}
\end{align}
In $(i)$ we have used Lemma~\ref{lemma:data.initialization}, so that $|\sip{x_i}{\mu_2}|\leq C_1 \sigma \sqrt d$ when $i\in \imt{+\mu_1}t$.  In inequality $(ii)$, we use that Lemma~\ref{lemma:data.initialization} implies $|\sip{x_i}{\mu_2}| \leq 2$ for $i\in \imt{\pm \mu_2}{t}$, so that by the 1-Lipschitz property of $\ell$ and $\phi$ we have,
\[ \l| \sum_{i\in \imnt {\pm \mu_2}t} -\ell'_{i,t} y_i \phi'(\sip{\wt t_j}{x_i})\sip{x_i}{\mu_2} \r| \leq 2|\calnt t|.\]
In inequality $(iii)$, we have used that $\ell$ is 1-Lipschitz as well as Lemma~\ref{lemma:data.initialization} so that $|\sip{x_i}{\mu_2}| \leq 3/2$ for $i\in \imt{-\mu_2}{t}$.

From the above, one can see that if $\sum_{i\in \imct {+\mu_2}t} \phi'(\sip{\wt t_j}{x_i}) \gg \sum_{i\in \imct {-\mu_2} t} \phi'(\sip{\wt t_j}{x_i})$, then we will have that the above quantity is negative, showing that $\sip{\wt t_j}{\mu_2}$ will decrease.  When $\sip{\wt t_j}{\mu_2}$ is large, then this is likely to occur; this is precisely the second part of Lemma~\ref{lemma:Nj+.argument.all.times}.  
In particular, since $\sip{\wt {t}_j}{\mu_2}\geq \alpha |a_j|$ by assumption, we have
\begin{align*}
    \ip{\f{\wt t_j}{\snorm{\wt t_j}}}{\mu_2} \overset{(i)}\geq \f{ \alpha |a_j|}{3 \alpha |a_j| t} \overset{(ii)} \geq \f{4}{3} \alpha \overset{(iii)}\geq \f{1}{2\sqrt{C}},\numberthis \label{eq:sip.wt.mu2.ao}
\end{align*}
where $(i)$ follows by Lemma~\ref{lemma:per.neuron.norm} and the fact that we are considering the case when $\sip{ \wt t_j}{\mu_2} \geq \alpha |a_j|$; inequality $(ii)$ uses that $t\leq 1/(4\alpha)$; and $(iii)$ uses Assumption~\ref{a:stepsize}, so that $\alpha \geq 1/(2\sqrt C)$. Since the correlation with the cluster mean is of constant order, we can repeat the argument used in~\eqref{eq:correlation.with.cluster.mean.to.samples} to show that the sign of $\sip{\wt t_j}{\xit it}$ is the same as the sign of $\sip{\wt t_j}{\mu_2}$ for $i\in \imct{+\mu_2}{t}$:
\begin{align*}
\sip{\wt t_j / \snorm{\wt t_j}}{\xit it} &\geq \sip{\wt t_j / \snorm{\wt t_j}}{\mu_2} - \snorm{\xit it - \mu_2} \\
&\overset{(i)}\geq \f 1 {2\sqrt C} - C_1 \sigma \sqrt d \\
&\overset{(ii)}>0.
\end{align*} 
Inequality $(i)$ uses the lower bound in~\eqref{eq:sip.wt.mu2.ao} as well as Lemma~\ref{lemma:data.initialization}.  Inequality $(ii)$ uses assumption~\ref{a:sigma}, so that $C_1 \sigma \sqrt d \leq C_1/C < 1/(2\sqrt C)$ for $C$ sufficiently large relative to $C_1$.  Using a symmetric argument, we thus have for positive neurons satisfying $\sip{\wt t_j}{\mu_2}\geq \alpha |a_j|$,
\begin{equation}
\text{for every $i\in \imct{+\mu_2}{t}$}, \quad \phi'(\sip{\wt t_j}{\xit it})=1,\quad \text{while for $i\in \imct{-\mu_2}{t}$},\quad \phi'(\sip{\wt t_j}{\xit it})=0.\label{eq:ao.intermediate.claim}
 \end{equation}
Substituting the above into~\eqref{eq:almost.orthogonality.mu2.intermediate}, we get,
\begin{align*}
     \sip{ w_{j}^{(t+1)} - w_{j}^{(t)}}{\mu_2} &\leq \alpha |a_j| \Bigg[  2C_1 \sigma \sqrt{d} + 2 \f{|{\calnt t}|}n \\
    &\quad - \f {\exp(-2)}{4n} \Bigg( \sum_{i\in \imct {+\mu_2}{t}} \phi'(\sip{\wt t_j}{\xit it}) - 6\exp(2)\cdot \sum_{i\in \imct {-\mu_2}{t}} \phi'(\sip{\wt t_j}{\xit it}) \Bigg)\Bigg] \\
    &\overset{(i)}\leq \alpha |a_j| \Bigg[ 2C_1 \sigma \sqrt{d} + 2 \f{|{\calnt t}|}n - \f {\exp(-2)}{4n} |\imct {+\mu_2}{t}|  \Bigg]\\
    &\overset{(ii)}\leq \alpha |a_j| \l[ 2 C_1 \sigma \sqrt{d} + 3 \f{|\calnt t|}n - \f{\exp(-2)}{4n} |\imt {+\mu_2}t| \r]\\
    &\overset{(iii)}\leq \alpha |a_j| \l[ 2 C_1 \sigma \sqrt{d} + 3 \f{|\calnt t|}n - \f{\exp(-2)}{32}  \r]\\
    &\overset{(iv)}<    0.\numberthis 
\end{align*}
The inequality $(i)$ uses eq.~\eqref{eq:ao.intermediate.claim}.  Inequality $(ii)$ uses that $|\imct {+\mu_2}t| \geq |\imt{+\mu_2}{t}| - |\imnt{+\mu_2}{t}|\geq |\imt{+\mu_2}{t}| - |\calnt t|$.   Inequality $(iii)$ uses the lower bound on the number of points in cluster $\mu_2$ given in Lemma~\ref{lemma:data.initialization} together with Assumption~\ref{a:samples}.  The final inequality follows by using Assumption~\ref{a:sigma} and Lemma~\ref{lemma:data.initialization}, which allow for us to take $\sigma \sqrt{d}$ and $|\calnt t|/n$ smaller than an absolute constant.  
This shows that, in the case that $\sip{\wt t_j}{\mu_2} \geq  \alpha |a_j|$, the value of $\sip{\wt {t+1}_j}{\mu_2}$ is strictly less than $\sip{\wt t_j}{\mu_2}$.  Since by Lemma~\ref{lemma:data.initialization} we have $\snorm{\xit it}\leq \sqrt 2$, we have,
\begin{equation} 
|\sip{\wt {t+1}_j - \wt t_j}{\mu_2}| = \alpha |a_j| \l| \f 1 n \summ i n -\ell'_{i,t} y_i \phi'(\sip{\wt t_j}{x_i})\sip{x_i}{\mu_2} \r| \leq 2 \alpha|a_j|.\label{eq:wjt+1.increment.amount}
\end{equation}
As we have shown $\sip{\wt t_j}{\mu_2} \geq \alpha |a_j|$, this implies $\sip{\wt {t+1}_j}{\mu_2} \in\l[ (1 - 2) \alpha |a_j|, \alpha |a_j|\r)$, and thus the inequality $|\sip{\wt {t+1}_j}{\mu_2}|\leq 3 \alpha |a_j|$ holds as desired.  This completes the induction in the case that $\sip{\wt t_j}{\mu_2} \geq \alpha |a_j|$.  

For the case $\sip{\wt t_j}{\mu_2} \leq - \alpha |a_j|$, we can use a nearly identical argument as above to show that $\newline\sip{\wt {t+1}_j~-~\wt t_j}{\mu_2}>0$ so that $\sip{\wt {t+1}_j}{\mu_2}\in (-\alpha |a_j|, (-1+ 2) \alpha |a_j|]$.  This again gives $|\sip{\wt {t+1}}{\mu_2}| \leq 3 \alpha |a_j|$. 

The only remaining case is when $|\sip{\wt t_j}{\mu_2}|\leq \alpha |a_j|$.  In this case,~\eqref{eq:wjt+1.increment.amount} implies that we have the inequality $|\sip{\wt {t+1}_j}{\mu_2}|\leq (1+2)\alpha |a_j| = 3 \alpha |a_j|$, completing the induction for the $J_{+\mu_1}$ neurons.  The proof that almost-orthogonality holds for neurons $j\in J_{-\mu_1}\cup J_{\pm \mu_2}$ holds using an analogous argument.
%\paragraph{Almost-orthogonality holds at time $t+1$: $J_{-\mu_1}\cup J_{\pm \mu_2}$ neurons.}
%To show almost-orthogonality holds for the $J_{-\mu_1}$ neurons, we can repeat the above argument by showing that, if $|\sip{\wt t_j}{\mu_2}|\geq \alpha |a_j|$, then at the following iteration either $\sip{\wt {t+1}_j}{\mu_2}$ will change sign or it will decrease in magnitude.  To show this, we derive an analogous version of equation~\eqref{eq:almost.orthogonality.mu2.intermediate}, and then use the fact that when $\sip{\wt t_j}{\mu_2} \geq \alpha |a_j|$, then Lemma~\ref{lemma:Nj+.argument.all.times} implies $\sum_{i\in \imct{+\mu_2}t} \phi'(\sip{\wt t_j}{x_i}) \gg \sum_{i\in \imct {-\mu_2}{t} }\phi'(\sip{\wt t_j}{x_i})$, which causes $\sip{\wt t_j}{\mu_2}$ to decrease.  The cases $\sip{\wt t_j}{\mu_2} \leq -\alpha |a_j|$ and $|\sip{\wt t_j}{\mu_2}| \leq \alpha |a_j|$ follow using the same arguments as in the $J_{+\mu_1}$ case.  The same logic holds for the $J_{\pm \mu_2}$ neurons. 

\end{proof}

\subsection{Proof of Theorem~\ref{thm:final.generalization}}\label{sec:mainthm}
For the reader's convenience, we restate the theorem below before completing its proof.
\mainthm*
\begin{proof}
% The proof boils down to choosing the step-size and initialization variance appropriately so that we may apply our previous results.  In particular, we claim that if we define
% \[ c:= \max \left\{ \f{12 \cdot 512 \exp(2)}{(1-1/C_0)^2}>0, (16 C_2)^{1/3} \right\}\]
% then for any choice of $\alpha$ and $\sinit$ satisfying
% \[ \f{1}{c^2} \leq \alpha \leq \f{1}{c},\quad \sinit \leq \f{1}{c^4 \sqrt{md}},\]
% the theorem holds.  We can then take the final constant $C$ given in the assumptions~\ref{a:dimension} through~\ref{a:noisy.fraction} to be larger than $c$ and the named constants $C_0, C_1, C_2$ from the previous lemmas so that the final theorem holds.
%To see where this choice of $c>0$ comes from,
First, note that with probability at least $1-4\delta$, a good run occurs, so that the results of Lemma~\ref{lemma:random.init}, Lemma~\ref{lemma:candidate.subnetwork},  Lemma~\ref{lemma:data.initialization}, and Lemma~\ref{lemma:Nj+.argument.all.times} all hold for the absolute constant $C_0=4^5 \cdot 1024^2 \exp(4)$.  
% In order to apply Lemma~\ref{lemma:nac.orthogonality.holds.all.times}, we require there exists some absolute constant $C_2'>1$ such that
% \begin{equation} \label{eq:nac.orthogonality.stepsize.sinit.requirement}
% \alpha \geq 1/C_2',\quad \text{and}\quad \sinit \sqrt{md} \leq \min(\alpha/3, 1/16C_2).
% \end{equation}
% By taking $C_2' = c^2$, we see that that for any $\alpha$ satisfying $1/c^2 \leq \alpha \leq 1/c$, we have
% \[ \sinit \sqrt {md}\leq 1/c^4 \leq \alpha/c^2 \leq \min(\alpha/3, 1/16C_2).\]
We thus can apply Lemma~\ref{lemma:nac.orthogonality.holds.all.times} so that neuron alignment and almost-orthogonality hold for times $t=1, \dots, 1/4\alpha$.  
% To apply Lemma~\ref{lemma:subnetwork.margin.time.t1}, we require
% \begin{equation} \label{eq:subnetwork.margin.stepsize.requirement}
% \alpha \leq \f{ \exp(-2)}{512 \cnac \cao} = \f{ \exp(-2)(1-1/C_0)^2}{12\cdot 512},
% \end{equation} 
% which is satisfied since $\alpha \leq 1/c$.  
Since neuron alignment and almost-orthogonality hold, by Lemma~\ref{lemma:subnetwork.margin.time.t1}, we have, 
 \begin{align*} 
\yit i{t} f^J(\xit i{t}; \Wt {T}) &\geq \f {\exp(-2)}{4\cdot 1024}(1-1/C_0)^2  \quad \text{ for all } i\in \calct {t},\quad \text{and}\\
\yit i{t} f^J(\xit i{t}; \Wt {T}) &\leq - \f {\exp(-2)}{4\cdot 1024} (1-1/C_0)^2 \quad \text{ for all } i\in \calnt  {t}.\label{eq:clean.subnetwork.margin} \numberthis 
\end{align*} 
In order to apply Lemma~\ref{lemma:good.subnetwork.suffices}, which relates the prediction on the subnetwork to the entire network, we need to ensure that for $C_f = 4\cdot 1024 \exp(2) / (1-1/C_0)^2$ we have $|J^c|/m \leq 1/16 C_f^2$. 
 If we denote by $J^c = [m]\setminus (J_{\pm \mu_1} \cup J_{\pm \mu_2})$, then $|J^c|/m \leq 1-(1-1/C_0)^2 \leq 2/C_0$, so that,
% \begin{equation} \label{eq:jc.ub}
% |J^c|/m \leq 1-(1-1/C_0)^2 \leq 2/C_0.
% \end{equation}
% Using~\eqref{eq:jc.ub}, we have,
\begin{equation*}% \label{eq:sufficient.eq.for.good.subnetwork.lemma}
\f{|J^c|}{m} \leq \f{2}{C_0}  \overset{(i)}= \f{\exp(-4)}{2 \cdot 16^2 \cdot 1024^2} \leq \f{\exp(-4)}{16^2 \cdot 1024^2} \l(1 - \f 1 {C_0}\r)^2 = \f{ 1}{16C_f^2}.
\end{equation*}
The equality $(i)$ follows since $C_0 = 4^5 \cdot 1024^2 \exp(4)$.  Thus we may apply Lemma~\ref{lemma:good.subnetwork.suffices}.  Since $\snorm{\Wt {T}}_F\leq 1$ by Lemma~\ref{lemma:subnetwork.margin.time.t1}, and since $\snorm{\xit i{t}}\leq 2$ by Lemma~\ref{lemma:data.initialization}, the lower bound for clean samples given in~\eqref{eq:clean.subnetwork.margin} can be used in Lemma~\ref{lemma:good.subnetwork.suffices} to get,
\begin{equation} \text{for all }i\in \calct {t},\quad  \yit i{t} f(\xit i{t}; \Wt {T}) \geq \f{ \exp(-2)}{16 \cdot 1024} =: \gamma >0.\label{eq:margin.clean.points}
\end{equation}
Using a symmetric argument, we have that noisy samples satisfy
\[\text{for all }i\in \calnt {t},\quad  \yit i{t} f(\xit i{t}; \Wt {T})\leq -\f{\exp(-2)}{16 \cdot 1024} = - \gamma < 0.\]
This shows that the neural network accurately classifiers all of the clean samples correctly at a margin of $\gamma>0$, and misclassifies all noisy samples incorrectly.  Since we have the Frobenius norm bound $\snorm{\Wt T}_F\leq 1$, we can therefore use a simple Rademacher complexity-based argument to derive a generalization bound for the neural network.  In particular, let us define the ramp loss
\[ r_\gamma(z) := \min(1, \max(0, 1 - z/\gamma)).\]
Then $r_\gamma$ is $\gamma^{-1}$-Lipschitz, and if we denote by
\[ \mathcal{F} := \{ x \mapsto f(x;W) : \snorm{W}_F\leq 1\}\]
as the class of two-layer ReLU networks with Frobenius norm at most 1, the expected Rademacher complexity~\citep[Lemma 26.9]{shalevschwartz} of the hypothesis class induced by the composition of $r_\gamma$ with the class of two-layer ReLU networks with Frobenius norm at most 1 satisfies
\[ \mathfrak{R}(r_\gamma \circ \mathcal{F}) \leq \gamma^{-1} \mathfrak{R}(\mathcal{F}).\]
Since $\E_{(x,y)\sim \pnoise} [\snorm{x}^2]\leq 2$, a standard bound on the Rademacher complexity of two-layer ReLU networks (see Proposition~\ref{prop:rademacher.complexity}) therefore implies
\[ \mathfrak{R}(r_\gamma \circ \mathcal{F}) \leq \f{ 4\gamma^{-1} }{\sqrt n}.\] 
Finally, note that by~\eqref{eq:margin.clean.points}, we have that the empirical risk under the ramp loss $r_\gamma$ is at most the risk under the zero-one loss,
\begin{equation}\label{eq:empirical.risk.ramp.loss}
    \f 1 n \summ i n r_\gamma(y_i f(x_i;\Wt t)) \leq \f{|\calN|}n.
\end{equation}
Standard Rademacher complexity generalization bounds
 (e.g.~\citet[Theorem 26.5]{shalevschwartz}) thus imply 
\begin{align*}
    P\big(y \neq \sgn(f(x; \Wt {T}))\big) &\leq \E[r_\gamma(y f(x;\Wt T)] \\
    &\leq \f 1 n \summ i n r_\gamma(y_i f(x_i;\Wt t)) + \f{ 4 \gamma^{-1}}{\sqrt n} + \sqrt{\f{2\log(4/\delta)}n} \\
    &\leq \f{ |\calN|}{n}  +  \f{ 4 \gamma^{-1}}{\sqrt n} + \sqrt{\f{2\log(4/\delta)}n}\\
    &\leq \eta + \f{ \sqrt{2C \log(2T/\delta)} + 4\gamma^{-1} + \sqrt{2 \log(4/\delta)}}{\sqrt n}.
\end{align*}
In the last inequality, we have used that $\snorm{\Wt{T}}_F\leq 1$ and that part (c) of Lemma~\ref{lemma:data.initialization} implies $|\calnt t|/n  \leq \eta + \sqrt{2C\log(1/\delta) / n}$. Since $T = 1/(4\alpha) +1$ and $\alpha \geq 1/(2\sqrt C)$, this completes the proof. 
\end{proof}

\section{Rademacher Complexity Bound} 
Below, we provide a characterization of the Rademacher complexity of the class of one-hidden-layer ReLU networks with weights that have a bounded Frobenius norm. 
\begin{proposition}\label{prop:rademacher.complexity}
Let $R>0$ be arbitrary, and let $\eps_i\iid \Unif(\{+1,-1\})$ be independent Rademacher random variables, and let $s\in \R^n$ denote the vector of Rademacher variables.  Consider $\hat {\mathfrak{R}}(\calF_R)$, the empirical Rademacher complexity~\citep{bartlett2003rademacher} of the function class 
\[ \calF_R := \{ x\mapsto f(x; W) : \snorm{W}_F \leq R \},\]
defined by 
\[ \hat {\mathfrak{R}}_n(\calF_R) := \E_{s \sim \Unif(\{+1,-1\})^n} \l[ \sup_{\snorm{W}_F\leq R} \f 1 n \summ i n \eps_i f(x_i;W) \r] .\]
Then, for $a_j \iid \Unif(\{1/\sqrt m, -1/\sqrt m\})$, we have
\[ \mathfrak{R}(\calF_R) := \E_{(x_i,y_i)\sim \calD^n} \hat{ \mathfrak{R} }(\calF_R) \leq \f{ 2R\sqrt{ \E_x \l[\snorm{x}^2\r]}}{\sqrt n}.\]
\end{proposition}
\begin{proof}
We mimic the proof given in~\cite[Lecture 8]{tengyuleturenotes}.   We have
\begin{align*}
    \hat{\mathfrak{R}}(\calF_R) &= \f 1 n \E_{\eps_i} \l[ \sup_{\snorm{W}_F\leq R} \summ i n \eps_i f(x_i;W) \r] \\
    &= \f 1 n \E_{\eps_i} \l[ \sup_{\snorm{W}_F\leq R} \summ i n \eps_i \summ j n a_j \phi(\sip{w_j}{x_i}) \r]\\
    &\overset{(i)}= \f 1 n \E_{\eps_i} \l[ \sup_{\snorm{W}_F\leq R} \summ j m  a_j \snorm{w_j}_2 \summ i n \eps_i \phi(\sip{w_j/\snorm{w_j}_2}{x_i}) \r]\\
    &\leq\f 1 n \E_{\eps_i} \l[ \l( \sup_{\snorm{W}_F\leq R} \summ j m  |a_j| \snorm{w_j}_2 \r) \max_{j\in [m]} \l| \summ i n \eps_i \phi(\sip{w_j/\snorm{w_j}_2}{x_i}) \r| \r]\\
    &\overset{(ii)}\leq\f R n \E_{\eps_i} \l[ \max_{j\in [m]} \l| \summ i n \eps_i \phi(\sip{w_j/\snorm{w_j}_2}{x_i}) \r| \r]\\
    &\leq \f R n \E_{\eps_i} \l[ \sup_{\norm{\bar w}\leq 1} \l| \summ i n \eps_i \phi(\sip{\bar w}{x_i}) \r| \r].\numberthis \label{eq:rademacher.intermediate.bound}
\end{align*}
In $(i)$ we use the homogeneity of the ReLU activation, and in $(ii)$ we use the Cauchy--Schwarz inequality to get that
\[ \summ j m |a_j| \snorm{w_j}_2 = \f 1{ \sqrt m} \summ j m \snorm{w_j}_2 \leq \f 1 {\sqrt m} \cdot \sqrt{m} \sqrt{\summ j m \snorm{w_j}^2} = \snorm{W}_F.\]
From~\eqref{eq:rademacher.intermediate.bound}, since $\phi$ is 1-Lipschitz and the zero function is included in the class $\{x\mapsto \phi(\sip{\bar w}{x}) : \snorm{\bar w}\leq 1\}$, a symmetrization argument yields~\cite[Lecture 5]{tengyuleturenotes}
\begin{align*}
    \hat{\mathfrak{R}}(\calF_R) \leq \f {2R}n  \E_{\eps_i} \l[ \sup_{\norm{\bar w}\leq 1} \summ i n \eps_i \phi(\sip{\bar w}{x_i}) \r] = 2 R  \cdot \hat {\mathfrak{R}} \l(\{x\mapsto \phi(\sip{\bar w}{x}) : \snorm{\bar w}\leq 1\}\r).
\end{align*}
Finally, as $\phi$ is 1-Lipschitz, the contraction property of the Rademacher complexity and standard Rademacher complexity bounds for linear hypothesis classes~\cite[Lemma 26.10]{shalevschwartz} yields the desired bound. 
\end{proof}

\section{Proof of Proposition~\ref{prop:feature.maps.different}}\label{sec:featuremapschange}
We restate and prove Proposition~\ref{prop:feature.maps.different} below.
\featuremapschange*
\begin{proof}
For simplicity, let us denote $x_i$ by the short-hand $x$.  Since the $j$-th component $(j\in [m])$ of $\phi(W x)$  is given by $\phi(\sip{w_j}{x})$, we have,
\begin{align*}
    \snorm{\phi(\Wt T x) - \phi(\Wt 0 x)} ^2 &= \summ j m [\phi(\sip{\wt T_j}{x}) - \phi(\sip{\wt 0_j}{x})]^2.
\end{align*}
To show that the feature map moves significantly, it therefore suffices to derive a lower bound on $|\phi(\sip {\wt T_j}{x}) - \phi(\sip{\wt 0_j}{x})|$ for each $j$.  To do so, we will show that for each sample $x$, a significant number of neurons have large, positive activations, so that $\sip{\wt T_j}{x} \gg 0$, while the near-zero initialization allows for us to essentially ignore the $\phi(\sip{\wt 0_j}{x})$ term. 

Since neuron alignment holds at times $t=1, \dots, T-1$, an identical argument to that of~\eqref{eq:margin.time.t.node.j} shows that for any $\mu \in \{\pm \mu_1, \pm \mu_2\}$ and $j\in J_{\mu}$, we have,
\[ \sip{\wt {T}_j - \wt 1_j}{\mu} \geq \f{ \alpha |a_j| (T-1)}{64} \exp(-2) = \f{ |a_j| \exp(-2)}{256}.\]
Moreover, using Equation~\eqref{eq:unnormalized.margin} we also have that $\sip{\wt 1_j - \wt 0_j}{\mu} >0$.  Adding this inequality to the preceding display, we get,
\begin{align*} 
\text{for each $j\in J_{+\mu_1}$,  }\quad &\sip{\wt {T}_j - \wt 0_j}{+\mu_1} \geq \f{ |a_j| \exp(-2)}{256},\\
\text{for each $j\in J_{-\mu_1}$,  }\quad &\sip{\wt {T}_j - \wt 0_j}{-\mu_1} \geq \f{ |a_j| \exp(-2)}{256},\\
\text{for each $j\in J_{+\mu_2}$,  }\quad &\sip{\wt {T}_j - \wt 0_j}{+\mu_2} \geq \f{ |a_j| \exp(-2)}{256},\\
\text{for each $j\in J_{-\mu_2}$,  }\quad &\sip{\wt {T}_j - \wt 0_j}{-\mu_2} \geq \f{ |a_j| \exp(-2)}{256}.\numberthis \label{eq:wtj-w0j.ip.lb}
\end{align*} 
Following an identical calculation used in the proof of Lemma~\ref{lemma:per.neuron.norm} (see Eq.~\eqref{eq:wtj-w0j.norm}), we know that $\snorm{\wt T_j - \wt 0_j} \leq \sqrt 2 |a_j| \alpha T=\sqrt 2 |a_j|(\alpha + 4)$. Since $\alpha\leq 1/10$ we thus have, 
\begin{equation}\label{eq:wtj-wt0j.norm.ub.lb}
    \text{for every $\mu \in \{\pm \mu_1, \pm \mu_2\}$ and each $j\in J_{\mu}$,}\quad \snorm{\wt T_j - \wt 0_j} \leq \f{8}{\sqrt m}.
\end{equation}
Let $\mu(x)\in \{\pm \mu_1, \pm \mu_2\}$ be such that $x\in I_{\mu(x)}$.   Then by Lemma~\ref{lemma:data.initialization}, we know that $\snorm{x - \mu(x)}\leq C_1 \sigma \sqrt d$, so that for any $j\in J_{\mu(x)}$,
\begin{align*}
    \sip{\wt T_j - \wt 0_j}{x} &=  \ip{\wt T_j - \wt 0_j}{\mu} + \ip{\wt T_j - \wt 0_j}{x - \mu} \\
    &\overset{(i)}\geq  \f{ \exp(-2)}{256 \sqrt m} -  C_1 \sigma \sqrt d \snorm{\wt T_j - \wt 0_j} \\
    &\overset{(ii)}\geq  \f{ \exp(-2)}{256 \sqrt m} -  \f{ 8C_1 \sigma \sqrt d }{\sqrt m} \\
    &\overset{(iii)}\geq \f{ \exp(-2)}{512 \sqrt m}.\numberthis \label{eq:wtj-w0j.x.ip.lb}
\end{align*}
In inequality $(i)$ we use~\eqref{eq:wtj-w0j.ip.lb} and $\snorm{x - \mu(x)}\leq C_1 \sigma \sqrt d$.  In inequality $(ii)$ we use~\eqref{eq:wtj-wt0j.norm.ub.lb}, and in inequality $(iii)$ we use Assumption~\ref{a:sigma} so that for $C>1$ sufficiently large, we have $8 C_1 \sigma \sqrt d \leq \exp(-2)/512$.  Since $\phi(z_1) - \phi(z_2) = z_1 - z_2$ when both $z_1>0$ and $z_2>0$, we thus have
\begin{align*}
    &\text{for all $i\in [n]$ and all $j\in J_{\mu(x_i)}$ satisfying $\sip{\wt 0_j}{x_i}>0$}, \\
    &\text{we have}\quad \phi(\sip{\wt T_j}{x_i}) - \phi(\sip{\wt 0_j}{x_i})\geq \f{\exp(-2)}{1024 \sqrt m}.\numberthis \label{eq:phi.difference.positiveinit}
\end{align*}

Now, note that by Lemma~\ref{lemma:data.initialization}, $\snorm{x} \leq 2$, and by Lemma~\ref{lemma:data.initialization}, we have $\snorm{\wt 0_j} \leq 2 \sinit \sqrt d$.  Continuing from~\eqref{eq:wtj-w0j.x.ip.lb}, we therefore have for any $j\in J_{\mu(x)}$,
\begin{align*}
    \sip{\wt T_j}{x} &\geq \f{\exp(-2)}{512 \sqrt m} - \sip{\wt 0_j}{x} \\
    &\geq \f{\exp(-2)}{512 \sqrt m} - 4 \sinit \sqrt{d} \\
    &\overset{(i)}\geq \f{\exp(-2)}{1024 \sqrt m},
\end{align*}
where inequality $(i)$ uses Assumption~\ref{a:sinit} so that for $C>1$ sufficiently large, we have $\sinit \leq \exp(-2)/(4096 \sqrt{md})$.  Since $\phi(\sip{\wt 0_j}{x})=0$ for $\sip{\wt 0_j}{x} <0$, this implies that
%SF: should be phi(<wt T_j, x_i>) - phi(<wt 0_j, x_i>) 
\begin{align*}
    &\text{for all $i\in [n]$ and all $j\in J_{\mu(x_i)}$ satisfying $\sip{\wt 0_j}{x_i}\leq 0$}, \\
    &\text{we have}\quad \phi(\sip{\wt T_j}{x_i}) - \phi(\sip{\wt T_j}{x_i})\geq \f{\exp(-2)}{1024 \sqrt m}.\numberthis \label{eq:phi.difference.zeroinit}
\end{align*}
Putting together~\eqref{eq:phi.difference.positiveinit} and~\eqref{eq:phi.difference.zeroinit}, we see that,
\begin{align*}
    \snorm{\phi(\Wt T x_i) - \phi(\Wt 0 x_i)} ^2 &=  \summ j m |\phi(\sip{\wt T_j}{x_i}) - \phi(\sip{\wt 0_j}{x_i})|^2 \\
    &\geq |J_{\mu(x_i)}| \cdot \l( \f{ \exp(-2)}{1024 \sqrt m}\r)^2 \\
    &\overset{(i)}\geq \f{ \exp(-4)}{1024^2} \cdot \f 1 4 \l( 1 - \f{1}{C_0}\r)^2 \\
    &\geq \f{ \exp(-4)}{8\cdot 1024^2}, \numberthis\label{eq:feature.change.lb}
\end{align*}
where inequality $(i)$ uses the lower bound on $|J_{\mu}|$ given in Lemma~\ref{lemma:candidate.subnetwork}.  

On the other hand, we have
\begin{align*}
    \snorm{\phi(\Wt 0 x_i)}  &\overset{(i)}\leq \snorm{\Wt 0}_F \snorm{x_i} \\
    &\overset{(ii)}\leq \f 3 2 \sinit \sqrt{md} \cdot 2,
\end{align*}
where $(i)$ uses that $\phi$ is 1-Lipschitz and $(ii)$ uses Lemma~\ref{lemma:data.initialization} and Lemma~\ref{lemma:per.neuron.norm}.  
Putting this upper bound together with~\eqref{eq:feature.change.lb}, we get,
\begin{align*}
    \f{ \snorm{\phi(\Wt T x_i) - \phi(\Wt 0 x_i)} }{\snorm{\phi(\Wt 0 x_i)}  } &\geq \f{ \exp(-2)}{16\cdot 1024  \sinit \sqrt{md} },
\end{align*}
completing the proof.  
% For the final claim, note that since $\phi$ is 1-Lipschitz, we have,
% \begin{align*}
%     \snorm{\phi(\Wt T x_i) - \phi(\Wt 0 x_i)}  &\leq \snorm{(\Wt T - \Wt 0 )x_i}_2 \leq \snorm{\Wt T - \Wt 0}_2\snorm{x_i}.
% \end{align*}
% We can thus use the change in the feature map as a lower bound for the change in the spectral norm:
% \begin{align*}
%     \f{ \snorm{\Wt T - \Wt 0}_F}{\snorm{\Wt 0}_F} &\geq \f{\snorm{\Wt T - \Wt 0}_2}{\snorm{\Wt 0}_F} \\
%     &\geq \f{ \snorm{\phi(\Wt T x_i) - \phi(\Wt 0 x_i)}}{\snorm{x_i} \snorm{\Wt 0}_F} \\
%     &\overset{(i)}\geq \f{ \exp(-2)}{ 4 \cdot 1024 \snorm{x_i} \snorm{\Wt 0}_F} \\
%     &\overset{(ii)}\geq \f{\exp(-2)}{32 \cdot 1024 \sinit \sqrt{md}}.
% \end{align*}
% Inequality $(i)$ uses equation~\eqref{eq:feature.change.lb}.  The final inequality $(ii)$ follows by using Lemma~\ref{lemma:per.neuron.norm}, so that $\snorm{\Wt 0}_F\leq \f 32 \sinit \sqrt{md}$, and by using Lemma~\ref{lemma:data.initialization} so that $\snorm{x_i}\leq 2$. 
% This completes the lower bound for the scaled distance traveled from initialization under the Frobenius norm.  Note that we could derive a similar lower bound for the distance under the spectral norm by utilizing standard concentration bounds for matrices with i.i.d. Gaussian entries. 
\end{proof}

\section{On the Optimal Error in the Noiseless Setting}\label{app:optimal.error}
In this section we show that in the noiseless setting $(\eta=0$), under assumptions~\ref{a:dimension} through~\ref{a:samples}, the optimal error achievable in $O(\sqrt{\log(1/\delta) / n})$ and that this test error is achieved by the classifier $x\mapsto \sgn(|\sip{\mu_1}{x}| - |\sip{\mu_2}{x}|)$.  

Denote $\nu(x) := |\sip{\mu_1}{x}| - |\sip{\mu_2}{x}|$.  By definition, the test error for the classifier induced by $\nu$ is
\begin{align*}
    \P(\sgn(\nu(x)) \neq y) &= \P(y \nu(x) < 0) \\
    &= \f 1 4 \Bigg(  \P_{z\sim \pclust} (\nu(z + \mu_1) < 0) + \P_{z\sim \pclust} (\nu(z-\mu_1) < 0) \\
    &\quad + \P_{z\sim \pclust} (-\nu(z + \mu_2) < 0) + \P_{z\sim \pclust} (-\nu(z-\mu_2) < 0) \Bigg). 
\end{align*}
We shall show that $\P_{z\sim \pclust} (\nu(z + \mu_1) < 0) = o_n(1)$, and an identical argument will yield the same bound for the remaining three terms.  By definition,
\begin{align*}
    \P_{z\sim \pclust} (\nu(z + \mu_1) < 0) &= \P_z( |\sip{\mu_1}{z+\mu_1}| - |\sip{\mu_2}{z+\mu_1}|) \\
    &= \P_z( |1 + \sip{\mu_1}{z}| - |\sip{\mu_2}{z}| < 0) \\
    &\leq \P_z ( |\sip{\mu_1}{z}| + |\sip{\mu_2}{z}| > 1) \\
    &\leq \P_z(|\sip{\mu_1}{z}| > 1/3) + \P_z(|\sip{\mu_2}{z}| > 1/3).\numberthis \label{eq:nu.mu1.ub}
\end{align*}
For $i\in \{1,2\}$, since $z\sim \pclust$ is log-concave with $\E[z]=0$, $\E[zz^\top]=\sigma^2 I$ and $\snorm{\mu_i}=1$, $\sip{\mu_i}{z/\sigma}$ is isotropic and log-concave and hence $\P_{z\sim \pclust}( |\sip{\mu_i}{z}| > 1/3) \leq 3\exp(-\sigma^{-1}/3)$  using \citet[Theorem 5.1 and Lemma 5.7]{lovasz}.  By assumption~\ref{a:sigma} this means $\P_{z\sim \pclust}( |\sip{\mu_i}{z}| > 1/3) \leq 3\exp(-C \sqrt d/3)$.  We claim that this quantity is at most $O(\sqrt{\log(1/\delta)/n})$ under assumption~\ref{a:dimension}.   To see this, note that $\exp(-C \sqrt d/3) \leq \sqrt{\log(1/\delta)/n}$ if and only if $C \sqrt d > \f 32 \log(\nicefrac n{\log(1/\delta)})$.  We have,
\begin{align*}
    \log^2 \l( \f{n}{\log(1/\delta)}\r) = \l( \log n - \log \f 1 \delta \r)^2
    &\leq \l( \log n + \f 1 \delta \r)^2 = \log^2(n/\delta).
\end{align*}
In particular, since by assumption~\ref{a:dimension} we have $d\geq C \log^2(n/\delta)$, we also have $d \geq C \log^2(n/\log(1/\delta))$.  In particular, $\sqrt d > \sqrt{C} \log(n/\log(1/\delta)) > \f 3 2 \log(n/\log(1/\delta))$ for $C \geq 2$. This shows that $\exp(-C \sqrt d/3) \leq \sqrt{\log(1/\delta)/n}$ and hence $\P_z(|\sip{\mu_i}{z}| > 1/3) = O(\sqrt{\log(1/\delta)/n})$ for each $i$.  Substituting this into~\eqref{eq:nu.mu1.ub} shows that 
\[ \P_{z\sim \pclust} (\nu(z + \mu_1) < 0) = O(\sqrt{\log(1/\delta)/n}).\] 

\section{Experiment details} \label{appendix:experiment}
We provide here the experimental details for Figure~\ref{fig:highd.vs.lowd}. We consider a two-layer ReLU network of the form~\eqref{eq:twolayerrelu} with $m=400$ neurons.  The within-cluster distribution is Gaussian, $\pclust \sim \mathsf N(0, \sigma^2 I_d)$, where the within-cluster variance is given by $\sigma^2 = 1/d^{1.2}$ and we flip 15\% of the labels within each cluster the orthogonal cluster's label.  We initialize using centered Gaussians with variance $\sinit^2 = \nicefrac{0.01}{md}$ and run with a step-size of $\alpha=0.1$.  Validation accuracy is measured using $n=6000$ samples.  

\printbibliography
\end{document}